\documentclass{article}

\usepackage[preprint]{neurips_2026}

\usepackage[utf8]{inputenc} 
\usepackage[T1]{fontenc}    
\usepackage[hidelinks]{hyperref}       
\usepackage{url}            
\usepackage{booktabs}       
\usepackage{amsfonts}       
\usepackage{nicefrac}       
\usepackage{microtype}      
\usepackage{xcolor}         
\usepackage{enumitem}
\usepackage{graphicx}
\usepackage{needspace}  

\usepackage{etoc}

\usepackage{amsmath}
\usepackage{amssymb}
\usepackage{mathtools}
\usepackage{amsthm}
\usepackage{bbm}
\DeclareRobustCommand{\mb}[1]{\boldsymbol{#1}}

\DeclareMathOperator*{\argmin}{arg\,min}

\renewcommand{\mid}{~\vert~}

\newcommand{\iid}[1]{\stackrel{\text{iid}}{#1}}

\newcommand{\mbb}{\mb{b}}
\newcommand{\mbc}{\mb{c}}

\newcommand{\mbe}{\mb{e}}

\newcommand{\mbh}{\mb{h}}

\newcommand{\mbm}{\mb{m}}

\newcommand{\mbu}{\mb{u}}
\newcommand{\mbv}{\mb{v}}

\newcommand{\mbx}{\mb{x}}
\newcommand{\mby}{\mb{y}}
\newcommand{\mbz}{\mb{z}}

\newcommand{\E}{\mathbb{E}}

\usepackage{wrapfig}
\usepackage{caption}   
\usepackage{algorithm}
\usepackage{algorithmic}
\newcommand{\LeftComment}[1]{\STATE  {$\triangleright$ #1}}

\usepackage[capitalize,noabbrev]{cleveref}

\theoremstyle{plain}
\newtheorem{theorem}{Theorem}[section]
\newtheorem{proposition}[theorem]{Proposition}

\theoremstyle{definition}

\theoremstyle{remark}

\usepackage{xspace} 
\newcommand{\methodname}{\texttt{kCGM}\xspace}
\newcommand{\CGM}{\texttt{CGM}\xspace}

\newcommand{\thetapre}{\theta_{\mathrm{pre}}}
\newcommand{\ppre}{p_{\thetapre}}
\newcommand{\ytarg}{\mby^{\mathrm{targ}}}
\newcommand{\logp}{\log p}

\title{
Calibrating Generative Models to Feature Distributions with MMD Finetuning
}

\author{%
  Nathaniel L.~Diamant \\
  Stanford University \\
  \texttt{diamant@stanford.edu} \\
  \And
  Brian L.~Trippe \\
  Stanford University \\
  \texttt{btrippe@stanford.edu} \\
}

\begin{document}

\maketitle

\begin{abstract}
Generative models can produce individually plausible samples while deviating substantially from a target set in the distribution of key features.
For example, a model pretrained on broad drug-like chemical space may generate molecules whose molecular features differ from those of a therapeutic class of interest, such as known antibiotics.
Correcting such distributional \emph{miscalibration} is challenging: direct finetuning on the target set can overfit and does not control which features are matched.
To fill this gap, we introduce kernel Calibrating Generative Models (\methodname).
\methodname minimizes a maximum mean discrepancy (MMD) between generated and target feature distributions using an unbiased score-function estimator, with KL regularization to remain close to the pretrained model.
On a target set of 174 antibiotics, direct finetuning sacrifices chemical validity for feature-distribution matching, whereas \methodname improves target feature matching while increasing validity.
We further demonstrate \methodname in protein and DNA generation tasks, showing it can adapt autoregressive, continuous-space diffusion, and discrete diffusion models using only feature-level supervision.
Code is available at \url{https://github.com/smithhenryd/cgm}.
\end{abstract}

\section{Introduction}
Generative models of molecules frequently suffer misalignment of the distribution of features of interest.
Consider antibiotic drug discovery: models trained on large, diverse small-molecule datasets can generate plausible molecules whose chemical-group distributions differ substantially from known antibiotics.
Finetuning directly on known antibiotics risks overfitting and cannot specify which features of the target set should be matched.
More broadly, aligning a generative model to a target feature distribution is a recurring goal in biomolecular generative modeling: protein generators fail to match natural secondary-structure distributions \citep{lu2025assessing},
and regulatory DNA generators fail to reproduce realistic activity-profile distributions \citep{sarkar2024designing}.
While aligning, or \emph{calibrating}, these models to match the distribution of such features to a reference set could enable more faithful modeling and design of biomolecules, satisfactory finetuning algorithms do not exist.

\begin{figure}[t]
\centering
\includegraphics[width=0.99\columnwidth]{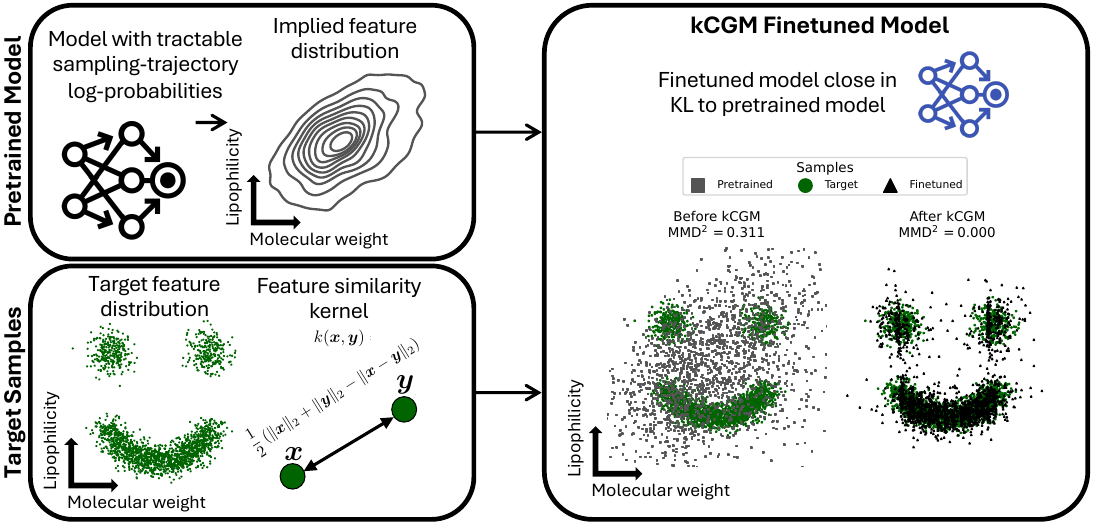}
    \caption{
\methodname overview. The user specifies a pretrained model with tractable 
sampling-trajectory log-probabilities (top left) and a target feature distribution through samples in feature space (bottom left). \methodname compares these distributions using a kernel similarity between feature samples and finetunes the model, with KL regularization toward the pretrained model. The result is a finetuned model whose generated samples more closely match the target feature distribution (right).
    }
    \label{fig:overview}
    \vspace{-0.4  cm}
\end{figure}

An ideal finetuning method for calibrating feature distributions should satisfy two properties.
First, it should work from feature-level supervision, matching selected aspects of a target distribution without requiring full target samples needed for direct finetuning.
Second, it should allow discrete-valued and black-box features such as the bit-vector Morgan fingerprints commonly used in small-molecule drug discovery.

To satisfy these goals, we introduce kernel Calibrating Generative Models (\methodname), a method for finetuning generative models to match target feature distributions from feature-space samples alone.
\methodname extends Calibrating Generative Models (\CGM) \citep{smith2025calibrating}, which finetunes a pretrained model so that feature \emph{means} match target values.
Like \CGM, \methodname applies when one can sample from the model and compute score-function gradients of the sampling trajectory, while allowing the feature map to be black-box or non-differentiable.
We show an overview of \methodname in \Cref{fig:overview}.
Rather than discretizing the target distribution as one would with \CGM, \methodname minimizes an unbiased finite-sample estimator of the kernel Maximum Mean Discrepancy (MMD), enabling more favorable scaling to higher-dimensional features and tighter agreement with the target distribution.

This paper presents three contributions.
First, we formulate feature-distribution matching as KL-regularized MMD finetuning from target feature samples alone.
Second, we derive an unbiased score-function gradient estimator for the resulting objective, enabling non-differentiable feature maps and non-reparameterizable samplers.
Finally, across protein, small-molecule, and regulatory DNA generation tasks, \methodname outperforms \CGM over the tested KL range and outperforms direct finetuning on the antibiotic task while improving chemical validity.

\subsection{Related work}

Several lines of work study how to adapt a generative model to a target population or task.
We discuss the most relevant: direct finetuning on samples, reward finetuning, \CGM, and prior uses of MMD in generative modeling.
No prior work has directly addressed the task of finetuning a generative model to match a target distribution over user-specified features.

\paragraph{Direct/supervised finetuning.}
A standard way to adapt a generative model is to continue likelihood training on a smaller, more relevant target dataset \citep{yosinski2014transferable}.
This two-stage procedure, often called supervised finetuning, is widely used to adapt pretrained models to task-specific data \citep{chen2021evaluating,ouyang2022training}.
In drug design, models are pretrained on large, broad molecular datasets and then finetuned on molecules with a desired activity~\citep{segler2018generating}.
While direct finetuning can make generations more target-like, it requires full target samples, can overfit small target sets, and cannot specify which aspects of the target distribution should be preserved.

\paragraph{Reward finetuning.} Reward finetuning methods adapt a generative model by assigning a scalar reward to each generated sample and optimizing the model to increase expected reward, e.g. \citet{christiano2017deep, rafailov2023direct, uehara2024fine, domingo2025adjoint, zheng2025diffusionnft}.
In contrast, \methodname is designed for settings where the goal is to match a target feature distribution rather than maximize a reward assigned to individual samples.

\paragraph{Calibrating Generative Models (CGM).}
\methodname generalizes CGM \citep{smith2025calibrating}, the prior method closest to the present work.
Whereas CGM matches target feature means, \methodname matches feature distributions through an MMD.
The special case of \methodname with the dot-product kernel recovers CGM-relax, a variant of CGM; Appendix~\Cref{app_subsec:kCGM_CGM} formalizes this observation and provides a proof.

\citet{smith2025calibrating} show \CGM can be extended toward distribution matching using a discretization of the empirical multivariate cumulative distribution function (CDF) of target set features.
However, this introduces discretization hyperparameters and, because the CDF representation grows exponentially with dimension, is impractical for features of dimension larger than three.
Moreover, in application this approach is shown to leave significant miscalibration error.

\paragraph{MMD in generative modeling.}
MMD has many uses in training and finetuning generative models, including DISCO Nets \citep{NIPS2016_disconets}, MMD-GAN \citep{li2017mmd}, and MMD-based losses for diffusion models \citep{de2025distributional,de2026flow}.
These works rely on reparameterization gradients, which are not directly applicable when sampling or the features of interest are non-differentiable.

Score-function estimators have also been used to optimize MMD objectives in non-reparameterizable settings.
\citet{masood2019diversity} use policy-gradient estimators for an MMD diversity term between policy trajectory distributions, and \citet{manenti2025leaning} use a score-function estimator for an MMD loss between stochastic predictive output distributions in graph structure learning.
\citet{miao2024training} use an MMD-based signal to finetune image diffusion models, but target diversity maximization rather than exact feature-distribution matching, and rely on a heuristic decomposition of set-level rewards into per-sample rewards.
In contrast, \methodname finetunes pretrained generative models to match externally specified target feature distributions, incorporates KL regularization to the pretrained model, and derives the finite-sample MMD coefficients and leave-one-out baselines needed for practical score-function optimization.

\section{Feature distribution matching with MMD}\label{sec:methods}
This section presents \methodname.
We first make precise our notation and assumptions and then present our MMD objective, regularization, and  finetuning algorithm.

\paragraph{Notation and assumptions.}
Let $p_{\theta_\mathrm{pre}}(\mbx)$ be a pretrained generative model with parameters $\theta_{\mathrm{pre}}$, and let $\mby=\mbh(\mbx)$ denote features of a generated sample in a set $\mathcal{Y}$.
In addition, we have examples of features
$
    \ytarg_{1:N} := (\ytarg_1,\dots,\ytarg_N),
$
from a target distribution.

Our goal is to choose $\theta$ so that if $\mbx \sim p_\theta$, then the induced feature distribution of $\mbh(\mbx)$ matches a target distribution.
The features may be extracted by a black-box, non-differentiable procedure.

We assume samples $\mbx \sim p_\theta$ and their log-probabilities may be readily drawn and computed, respectively;
this assumption enables score-function gradients \citep{williams1992simple}.
Here, $\mbx$ may denote the full sampling trajectory rather than only the terminal generated object.
For an autoregressive model, $\mbx$ is the generated token sequence and $\logp_\theta(\mbx)$ is the usual autoregressive log-probability.
For a diffusion sampler, $\mbx$ includes the intermediate reverse-process states, so $\logp_\theta(\mbx)$ decomposes into transition log-probabilities along the sampled denoising path.
The feature map $\mbh$ acts on the terminal generated sample of the trajectory, such as the fully denoised output of a diffusion sampler.
Thus the score-function estimator does not require differentiating through $\mbh$ or the sampling procedure.


\paragraph{MMD objective and the kernel choice.}
To compare feature distributions from samples, we use a maximum mean discrepancy (MMD).
For distributions $P$ and $Q$ and a positive definite kernel $k(\cdot,\cdot)$, the population squared MMD is
\begin{equation}\label{eqn:pop_MMD}
\operatorname{MMD}^2_k(P,Q)
=
\mathbb{E}_{\mby,\mby' \sim P}[k(\mby,\mby')]
- 2\mathbb{E}_{\mby \sim P, \mby' \sim Q}[k(\mby,\mby')]
+ \mathbb{E}_{\mby,\mby' \sim Q}[k(\mby,\mby')] .
\end{equation}
In our setting, $P$ is the model-induced feature distribution of $\mbh(\mbx)$ for $\mbx \sim p_\theta$, and $Q$ is the empirical distribution of target sample features $\ytarg_{1:N}$.



The kernel $k$ determines how differences between feature distributions contribute to the MMD.
For example, the MMD defined by the
dot-product kernel $k(\mbu,\mbv) = \mbu^\top \mbv$
compares only means.
In fact, with the dot-product kernel, $\operatorname{MMD}^2_k(P,Q)$ exactly recovers the \CGM-relax objective
and so does not distinguish distributions with identical means, even if their higher-order structure differs.
Appendix~\Cref{app_subsec:kCGM_CGM} provides details.

By contrast, we consider kernels for which the MMD defines a metric on probability distributions that is equal to zero only for identical distributions.
For continuous Euclidean features, we use the energy-distance kernel,
\begin{equation}\label{eq:energy_kernel}
k(\mbu, \mbv) = \frac{1}{2}\left(\|\mbu\|_2 + \|\mbv\|_2 - \|\mbu-\mbv\|_2 \right).
\end{equation}
In our small-molecule experiments with bit-vector fingerprints as features (\Cref{subsec:g2pt}), we use the Tanimoto / Jaccard kernel \citep{bouchard2013proof}, 
\begin{equation}\label{eq:tanimoto_kernel}
k(\mbu, \mbv) = \frac{\mbu^\top \mbv}{\|\mbu\|_2^2 + \|\mbv\|_2^2 - \mbu^\top \mbv}.
\end{equation}
Both kernels have the advantage that they do not introduce an additional bandwidth hyperparameter, unlike radial basis function kernels for example.
\citet[Chapter 2]{duvenaud2014automatic} provides an overview of common kernel choices that might be preferable in other settings.

\paragraph{Training with the MMD objective.}
An attractive feature of MMD for finetuning feature distributions is that the expectations in the population estimand in \Cref{eqn:pop_MMD} may be approximated without bias by sample averages.
The squared MMD has the unbiased estimator
\begin{equation}
\label{eq:sample_MMD}
\begin{aligned}
    \widehat{\operatorname{MMD}}^2_k[\mby_{1:M}, \ytarg_{1:N}]
    = \;& \underbrace{
        \frac{1}{M(M-1)} \sum_{m'\neq m} k(\mby_m, \mby_{m'})
    }_{\text{Discourage mode collapse}}
 - \underbrace{
        \frac{2}{MN}\sum_{m=1}^M\sum_{n=1}^N k(\mby_m, \ytarg_n)
    }_{\text{Encourage similarity to target features}}
    + C_{\mathrm{targ}},
\end{aligned}
\end{equation}
where
$C_{\mathrm{targ}} =
\frac{1}{N(N-1)} \sum_{n' \neq n}
k(\ytarg_n, \ytarg_{n'})$
depends only on target features, and so may be ignored in the optimization.

\begin{wrapfigure}[28]{r}{0.48\textwidth} 
\vspace{-1.1em}
\captionsetup{type=algorithm} 
\caption{\methodname finetuning} \label{alg:kcgm} \vspace{-0.5em} 
\noindent\rule{\linewidth}{0.8pt} 
\vspace{-1em} 
\begin{algorithmic}
    \STATE \textbf{Input}: Pretrained model $\ppre$, feature map $\mbh(\cdot)$, target features $\ytarg_{1:N}$, kernel $k(\cdot,\cdot)$, batch size $M$, KL weight $\lambda$.
    \vspace{0.07cm}

    \LeftComment{Initialize and optimize}
    \STATE $p_\theta \leftarrow \ppre$
    \vspace{0.2cm}

    \WHILE{not converged}
        \LeftComment{Sample and compute KL coefficients}
        \STATE $\mbx_1,\dots,\mbx_M \overset{\mathrm{i.i.d.}}{\sim} p_{\texttt{stop-grad}(\theta)}$
        \STATE $\begin{aligned}[t]
            l_m \leftarrow {}&
            \bigl{(} \logp_{\texttt{stop-grad}(\theta)}(\mbx_m) \\
            &{}- \logp_{\theta_{\mathrm{pre}}}(\mbx_m)\bigr{)}
        \end{aligned}$
        \STATE $c_m^{\mathrm{KL}} \leftarrow \frac{1}{M} \texttt{KL-LOO}_m(l_{1:M})$

        \vspace{0.2cm}
        \LeftComment{Kernel discrepancy coefficients}
        \STATE $\mby_m \leftarrow \mbh(\mbx_m)$
        \STATE $s_m \leftarrow \frac{1}{M-1}\sum_{m' \neq m} k(\mby_m,\mby_{m'})$
        \STATE $r_m \leftarrow \frac{1}{N}\sum_{n=1}^N k(\mby_m,\ytarg_n)$
        \STATE $c_m^{\mathrm{k}} \leftarrow \frac{2}{M} \bigl(s_m - r_m\bigr)$
        \STATE $c_m^{\mathrm{k,LOO}} \leftarrow \texttt{MMD-LOO}_m(c_{1:M}^{\mathrm{k}})$

        \vspace{0.2cm}
        \LeftComment{Total coefficients and update}
        \STATE $c_m \leftarrow \lambda c_m^{\mathrm{KL}} + c_m^{\mathrm{k,LOO}}$
        \STATE $\widehat{\mathcal L}^{\mathrm{\methodname\text{-}surrogate}} \leftarrow \sum_{m=1}^M c_m \logp_\theta(\mbx_m)$
        \STATE $\theta \leftarrow \textrm{gradient-step}(\theta, \nabla_\theta \widehat{\mathcal L}^{\mathrm{\methodname\text{-}surrogate}})$
    \ENDWHILE
\end{algorithmic}
\end{wrapfigure}

\paragraph{Score gradients of the MMD loss with leave-one-out baseline to reduce variance.}
Relating the sample-based discrepancy estimate in \Cref{eq:sample_MMD} to practical gradient estimates requires care.
Reparameterization gradients can provide low-variance gradient estimates when both $\mbh$ and the sampling procedure are differentiable. 
However, in the present setting where we assume only that we can tractably sample and compute $\log p_\theta(\mbx)$, score-function gradients are a natural choice.

\Cref{alg:kcgm} and \Cref{app_sec:kCGM_derivations} present and derive the score-gradient updates.
A key component is the inclusion of a novel \emph{leave-one-out baseline}, a Monte Carlo control variate used to reduce variance in score-function gradient estimates \citep[Ch. 13]{sutton2018reinforcement}.
Standard leave-one-out baselines use the average coefficient from the other samples in the batch \citep{kool2019buy}.
However, this approach is not directly applicable to the MMD loss because the coefficients are coupled across samples, so directly averaging them would introduce bias.
Intuitively, the baseline should approximate the average coefficient that would have been obtained had $\mbx_m$ not appeared in the minibatch.
\Cref{alg:kcgm} uses a baseline constructed along this intuition, denoted \texttt{MMD-LOO}, and \Cref{app_subsec:loo} provides the mathematical expression and derivation.

\paragraph{KL regularization and the complete algorithm.}
In addition to the squared MMD loss, we regularize KL divergence to the pretrained model $D_{\operatorname{KL}}(p_\theta || \ppre)$ with weight hyperparameter $\lambda$ to ensure the finetuned model does not deviate too far from the pretrained model.
The combined objective is
\begin{equation}
    \label{eq:combined_objective}
    \begin{aligned}
        \mathcal{L}_\theta[k, \ytarg_{1:N}]
        &=
        \E_{\mbx_{1:M} \overset{i.i.d.}{\sim} p_{\theta}}
        \left[
            \widehat{\operatorname{MMD}}^2_k[\mby_{1:M}, \ytarg_{1:N}]
        \right]
        + \lambda \operatorname{D}_{\mathrm{KL}}(p_\theta \,\|\, \ppre).
    \end{aligned}
\end{equation}
KL regularization is a standard approach in reward finetuning \citep[]{fan2023dpok, uehara2024fine} and is used by \CGM.
As the KL estimator decomposes into independent per-sample terms, we apply the standard leave-one-out centering baseline to the KL coefficients, denoted \texttt{KL-LOO} in \Cref{alg:kcgm}.
We find varying $\lambda$ enables exploring the tradeoff between matching the target feature distribution and preserving the pretrained model's distribution.

The final algorithm is shown in \Cref{alg:kcgm}.
It consists of an optimization loop, wherein \textbf{(1)} samples are taken from the current model weights without tracking gradients, \textbf{(2)} the MMD and KL objectives are decomposed into per-sample coefficients with leave-one-out baselines, and \textbf{(3)} the coefficients are used to update the model weights with score-function gradients.
While the MMD computation is $O(M^2 + MN)$, it is negligible compared to model gradient calculations.

\section{Experiments}
We evaluate how well \methodname matches the target feature distribution in three applications with diverse feature maps.
Across these cases, \methodname substantially increases agreement with the distribution of target features and outperforms \CGM and direct finetuning baselines.
Before describing each experiment, we outline a shared protocol with additional details in Appendix~\Cref{app:additional_details}.

\paragraph{Evaluation details shared across experiments.}
Because both \methodname and \CGM include a KL regularizer to the pretrained model weighted by $\lambda$, no single $\lambda$ value determines a method's performance.
Instead, each method gives a tradeoff curve between distance from the pretrained model (measured as KL-to-pretrained) and distribution-matching error.
In practice, $\lambda$ must be chosen by considering this tradeoff.

The distribution-matching metric differs between experiments and is specified in each subsection.
Error bars are reported as the standard error of the mean from three replicates with different random seeds.

\paragraph{Baselines.}
Our primary baseline is \CGM, shorthand for the \CGM-relax variant of \citet{smith2025calibrating}, the closest prior method in our setting.
In the protein experiment (\Cref{subsec:genie}), the 2D feature space makes the CDF-binning variant of \CGM tractable, and we use it as described by \citet{smith2025calibrating}.
In the small-molecule and DNA experiments (\Cref{subsec:g2pt,subsec:dna}), binning is impractical in the higher-dimensional feature spaces, so we compare against using \CGM to target the mean of the feature distribution.
The small-molecule experiment additionally includes direct finetuning on the target set as a baseline.

\paragraph{Compute.} All experiments run on a single L40S GPU and range from one to 30 hours in duration per finetuning run.
Further details are provided in Appendix~\ref{app:additional_details}.


\subsection{Matching antibiotic feature distributions}\label{subsec:g2pt}
We first apply \methodname to calibrate a generative model of small molecules to match properties of a set of 174 small-molecule antibiotics.
\methodname delivers stronger agreement with chosen features than \CGM while maintaining sample quality.
By contrast, direct finetuning degrades model quality and leads to a high proportion of invalid molecules.

\paragraph{Background.} Small-molecule generative models are increasingly important in early-stage drug discovery \citep{du2024machine}, where they can help explore the vast space of synthesizable, drug-like molecules.
While the space of drug-like molecules may contain $10^{60}$ molecules \citep{polishchuk2013estimation}, molecule sets with a desired function are often small.

The number of known antibiotics, and especially the number of unique antibiotic scaffolds, is limited \citep{FARHA2025102562}. This makes antibiotics a natural target for \methodname.
Rather than finetuning a generator to reproduce the target set molecule-by-molecule, we ask whether it can be adapted to match selected \emph{aspects} of the antibiotic distribution.

\paragraph{G2PT model.}
We use G2PT, an autoregressive graph generative model in which molecular graphs are flattened into sequences of node and edge tokens \citep{chen_g2pt}.
Specifically, we use the one-million-parameter G2PT-small model pretrained on GuacaMol \citep{brown2019guacamol}, a benchmark of 1.3M drug-like molecules filtered from ChEMBL.
Thus, pretraining provides a broad prior over drug-like chemistry, while the antibiotic set represents a smaller and specialized target distribution.

\paragraph{Antibiotic features and kernels.}
We target a set of 174 small-molecule antibiotics curated by \citet{scalia2025deep}.
An advantage of \methodname is that the user can choose which aspects of the target distribution to match.
We consider three complementary feature spaces, capturing local chemical substructures, broader scaffold-level organization, and coarse physicochemical properties.

For local substructure, we use Morgan fingerprints \citep{rogers2010extended}, sparse binary vectors indicating the presence of hashed local subgraphs, paired with the Tanimoto kernel.

For scaffold-level organization, we use generic Murcko scaffolds, which ignore atom type and bond order to capture only the topology of the molecular scaffold \citep{bemis1996properties}.
This captures a different notion of similarity from Morgan fingerprints, as two molecules can share scaffolds while differing in local substructure, or share local substructures while differing in scaffold organization.
We embed generic scaffolds using the representation-learning strategy of Fr\'echet ChemNet Distance, taking the penultimate-layer activations of a molecular property predictor \citep{preuer2018frechet}.
Appendix~\Cref{fig:murcko_match} shows that scaffold embedding space separates molecules with different scaffolds, whereas Morgan fingerprints do so very weakly.

For coarse properties, we use a vector of eleven standard molecular descriptors (full list in Appendix~\Cref{subsec:abx_app}) summarizing physicochemical, structural, and drug-likeness properties.
We use the energy-distance kernel for both scaffold embeddings and descriptors.

\begin{figure}[h]
  \centering
  \includegraphics[width=\textwidth]{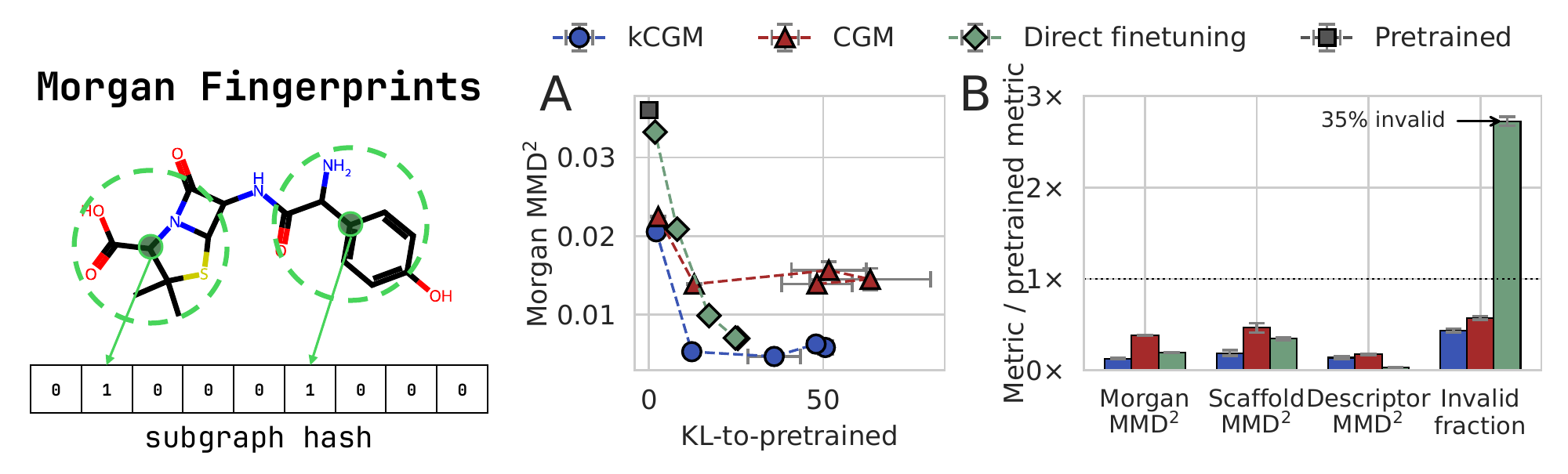}
  \caption{
Comparison of direct finetuning, \methodname, and \CGM for adapting G2PT to the antibiotics set. Panel A evaluates \methodname and \CGM models trained to match the target Morgan fingerprint distribution compared to direct finetuning according to MMD$^2$ over five values of the regularization strength, $\lambda$. Panel B selects the best $\lambda$ for each method according to Morgan $\operatorname{MMD}^2$. It then evaluates the selected models using scaffold $\operatorname{MMD}^2$, descriptor $\operatorname{MMD}^2$, and the fraction of invalid generated molecules.
  }
  \vspace{-1.0em}
  \label{fig:abx_results}
\end{figure}

\paragraph{Distribution matching results.}
We first focus on Morgan fingerprints, a common representation of local chemical substructure.
This setting is a natural test for \methodname because Morgan fingerprints are discrete bit vectors, making the feature map non-differentiable, and because the Tanimoto kernel provides a chemically meaningful similarity measure on the resulting feature space.

As baselines, we compare against \CGM, which matches only the target Morgan fingerprint mean, and direct finetuning on the 174 antibiotics.
Direct finetuning uses the standard autoregressive log-likelihood objective on full target molecules in the original sample space, with the same KL-to-pretrained regularization used for \methodname and \CGM; details are given in Appendix~\Cref{subsec:abx_app}.
For each method, we sweep five values of $\lambda$ and evaluate the tradeoff between Morgan fingerprint distribution matching, measured only on chemically valid generations, and distance from the pretrained model.

\methodname achieves a better Morgan $\operatorname{MMD}^2$--KL tradeoff than both \CGM and direct finetuning (\Cref{fig:abx_results}A).
This shows that matching only the average fingerprint, as in \CGM, is insufficient for matching the antibiotic fingerprint distribution.
Direct finetuning can also reduce Morgan $\operatorname{MMD}^2$, but does so by moving farther from the pretrained model and, as discussed below, at a cost to sample validity.

We next ask whether Morgan fingerprint finetuning improves other molecular feature distributions.
For each method, we select the best $\lambda$ according to Morgan $\operatorname{MMD}^2$ and evaluate the resulting models on scaffold, descriptor, and validity metrics (\Cref{fig:abx_results}B).
Although \methodname is trained only on Morgan fingerprints, it also improves scaffold and descriptor MMD relative to the pretrained model, suggesting that local substructure matching transfers partially to broader molecular properties.
In contrast, direct finetuning produces a high fraction of chemically invalid graphs.
This occurs because direct finetuning imitates the small antibiotic set molecule-by-molecule and does not explicitly discourage invalid generations, whereas invalid molecular graphs are mapped to a default feature value and are therefore penalized by the \CGM and \methodname objectives.
In addition, all \methodname runs maintain greater than 98\% uniqueness among generated molecules, indicating that the improved feature matching does not arise from memorization.

Finally, we repeat the same comparison when \methodname and \CGM are trained directly on scaffold embeddings or molecular descriptors.
These additional feature-targeted experiments are deferred to Appendix~\Cref{fig:abx_other_features} and show that \methodname remains effective across multiple user-chosen notions of molecular similarity.
A KL-matched summary of best interpolated metric values and cross-feature transfer is provided in Appendix~\Cref{tab:g2pt_full}.

\begin{wrapfigure}{r}{0.45\columnwidth}
\vspace{-1.0em}
\centering
\includegraphics[width=1\linewidth]{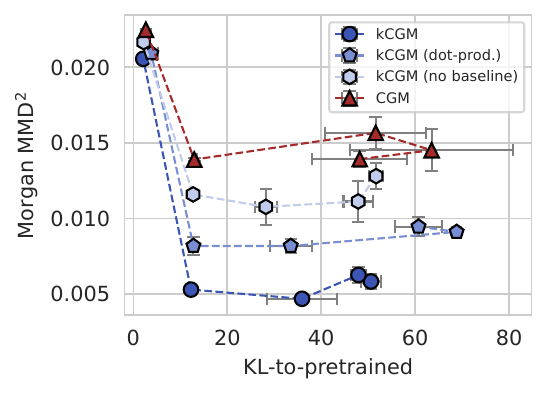}
    \caption{
Ablations of \methodname\ with Morgan fingerprint features. We compare standard \methodname, \methodname\ with the dot-product kernel, \methodname\ without the leave-one-out baseline, and \CGM.
    }
    \label{fig:g2pt_ablation}
\vspace{-2.5em}
\end{wrapfigure}

\paragraph{Ablation Results.}
We next ablate two components of \methodname\ on the Morgan fingerprint experiment.
First, we replace the Tanimoto kernel with the dot-product kernel, which reduces \methodname\ to mean matching rather than full distribution matching.
Second, we remove the leave-one-out baseline used for variance reduction.
Both changes degrade performance (\Cref{fig:g2pt_ablation}).
Using the dot-product kernel substantially increases MMD, confirming that distribution rather than mean matching is important.
Removing the leave-one-out baseline also substantially worsens \methodname, showing that the baseline is important for effective optimization.
\methodname\ with the dot-product kernel can be viewed as \CGM\ with our leave-one-out baseline, and substantially outperforms \CGM\ in this experiment.
This suggests the leave-one-out baseline is a drop-in improvement for \CGM.

\subsection{Increasing protein structure diversity}\label{subsec:genie}

We adopt an experiment from \citet{smith2025calibrating} in which Genie 2 \citep{lin2024out}, a protein structure diffusion model, is finetuned so that its secondary structure distribution matches a diverse target distribution of secondary structural proportions taken from protein domains in the CATH dataset \citep{sillitoe2021cath}.
While \citet{smith2025calibrating} demonstrated improved agreement with CATH, substantial disagreement remained.
With \methodname, the disagreement is roughly 2$\times$ smaller than that of \CGM.
In these experiments, we additionally explore a \emph{self-repulsion weight} hyperparameter that we find may be tuned to improve performance.

\begin{figure}[t]
  \begin{center}
    \centering
    \includegraphics[width=\textwidth]{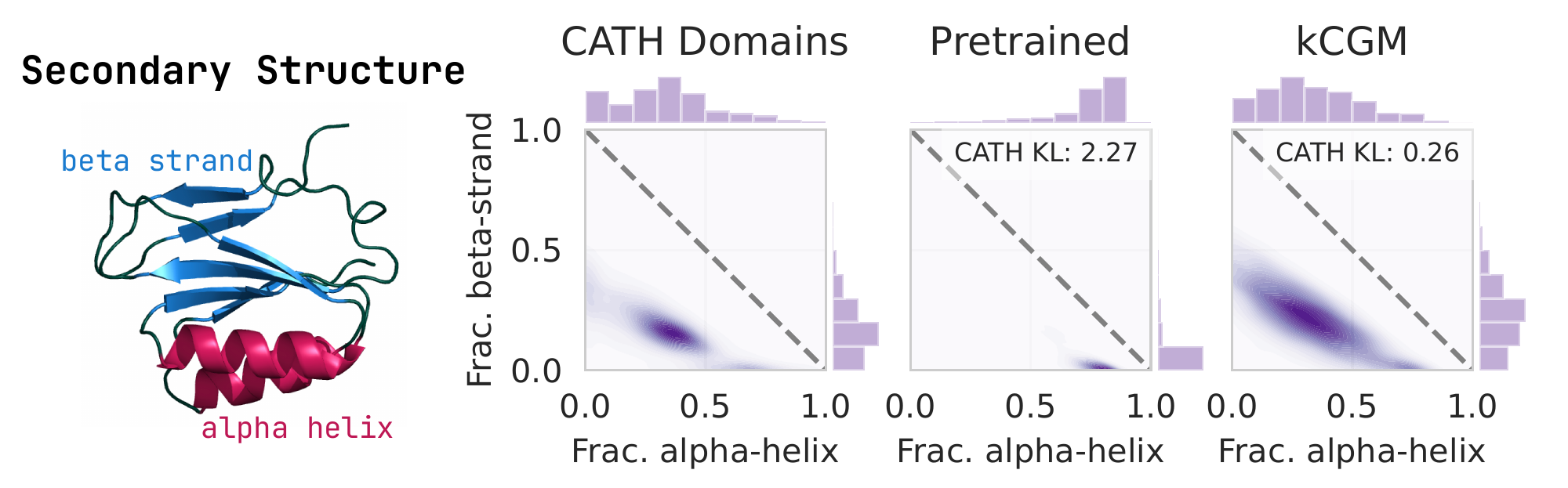}
    \caption{
Joint KDEs and marginal histograms of secondary-structure features for the target CATH domains, the pretrained Genie 2 model, and the best \methodname model. The pretrained model is concentrated near all-$\alpha$ structures, whereas \methodname shifts the distribution toward the diverse target distribution.}
    \label{fig:genie2_kdes}
  \end{center}
    \vskip -0.2in
\end{figure}

\paragraph{Background and setup.}
Due to a sampling heuristic necessary for high-quality generations, Genie 2 produces samples that consist almost entirely of alpha-helices.
In contrast, CATH protein domains exhibit high structural diversity and thus present a useful target distribution.

Following \citet{smith2025calibrating}, we generate protein domains with 100 residues and define the feature function as the proportion of the residues in alpha-helices and beta-strands:
\[
    \mbh(\mbx) = (\text{\# $\alpha$-helix residues}/100, \text{\# $\beta$-strand residues} /100).
\]
\methodname matches this 2D feature distribution directly, using the energy-distance kernel (\Cref{eq:energy_kernel}) which has no additional hyperparameters.
\CGM, by contrast, cannot use this 2D feature directly.
Its mean-matching objective would only constrain the average helix and strand proportions, not the shape of the joint distribution.
To recover distributional information, \citet{smith2025calibrating} discretize the CATH distribution into 99 bivariate CDF bins and constrain \CGM to match the proportion of samples falling in each bin.
In effect, \methodname matches the full 2D distribution through sample-based kernel comparisons, whereas \CGM matches 99 scalar bin-proportion means.
Because the feature space is only 2D, the CDF-binning strategy is tractable, making this a favorable setting for \CGM.
We reuse the \CGM hyperparameters (batch size, learning rate, etc.) without hyperoptimizing them for \methodname.

\begin{wrapfigure}[21]{r}{0.5\columnwidth}
\vspace{-0.5em}
\centering
\includegraphics[width=\linewidth]{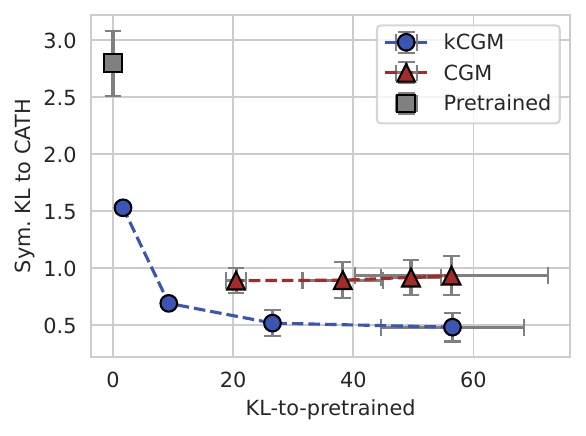}
\caption{
Distribution-matching error versus KL-to-pretrained for the protein secondary-structure experiment for four values of the KL-regularization weight $\lambda$. \methodname\ with $\alpha=0.5$ achieves lower symmetrized KL to the target CATH secondary-structure distribution than \CGM across the tested tradeoff range.
}
\label{fig:genie2_results}
\end{wrapfigure}

\paragraph{Self-repulsion hyperparameter.}
We additionally explore a \emph{self-repulsion} weight hyperparameter $\alpha$, which we include as a multiplicative factor on the first term in the MMD estimate in \Cref{eq:sample_MMD}, introducing a mode-seeking bias to the objective.
This results in a change to the MMD coefficients in \Cref{alg:kcgm},
$c_m^{\mathrm{k}} \leftarrow \frac{2}{M} \bigl(\alpha s_m - r_m\bigr)$.
We search three $\alpha$ values ($0.25, 0.5, 1.0$), though this is not necessary to outperform the baseline.
\Cref{app_subsec:cat_self_wt} provides greater detail, as well as examples to provide intuition for the impact of $\alpha.$

\paragraph{Results.}
\Cref{fig:genie2_results} shows the tradeoff curve between KL-to-pretrained and matching error, measured as the symmetric KL between Gaussian KDEs of the 2D features (\Cref{fig:genie2_kdes}).
\methodname achieves lower matching error than \CGM at every KL-to-pretrained budget we tested.
Moreover, \methodname's matching error varies smoothly with $\lambda$, whereas \CGM's matching error is relatively insensitive to $\lambda$, making it difficult for a practitioner to optimize the tradeoff.

While sweeping over $\alpha$ further improves distribution matching (Appendix~\Cref{fig:genie2_results_full}), every $\alpha$ value we test outperforms the baseline and maintains a controlled regularization curve.
This shows that \methodname is robust to both the regularization weight and the self-repulsion weight and reliably improves distribution matching compared to the baseline.

\subsection{Generating DNA with realistic cell-type-specific activity profiles}\label{subsec:dna}

Regulatory DNA sequence generation provides a challenging test of feature-distribution matching.
A useful enhancer generator should not only produce sequences that resemble the training set, but should also reproduce the distribution of regulatory activity associated with each biological condition \citep{lal2025polygraph}.
This is difficult to enforce with standard likelihood training and even large DNA generative models often fail to learn regulatory syntax \citep{patel2024dart}.

\paragraph{Background and setup.} We study this setting using the human melanoma enhancer DeepMEL2 dataset \citep{atak2021interpretation}.
The dataset consists of 500-base-pair regulatory sequences annotated with 47 binary cis-regulatory topics derived from melanoma ATAC-seq data.
These topics summarize melanoma cell states and provide multi-label conditioning variables.
We train a convolutional masked discrete diffusion model \citep{sahoo2024simple} to generate 500-base-pair DNA sequences conditional on these topic vectors.
For finetuning, we construct a separate target feature distribution for each topic using real training sequences positive for that topic.
Thus, when the model is sampled conditionally on a given topic, \methodname compares the generated AlphaGenome feature distribution to the target AlphaGenome feature distribution for that same topic.
Appendix~\Cref{app_subsec:dna} describes the model architecture, target construction, and training procedure.
Whereas sequences with high levels of predicted activity are rare in pretrained model samples, after \methodname finetuning DNA samples exhibit similar activity profiles to real enhancers with the corresponding topic annotations.

\begin{wrapfigure}[18]{r}{0.45\columnwidth}
\vspace{-1.3em}
\centering
\includegraphics[width=1\linewidth]{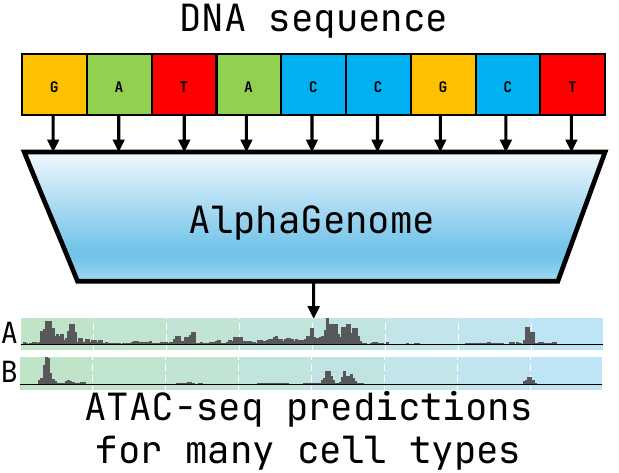}
    \caption{
Diagram of AlphaGenome activity features for \methodname finetuning. The maximum height of each ATAC-seq track is used to create a vector of maximum accessibility for each AlphaGenome cell type.
    }
    \label{fig:alpha_genome_feat}
\vspace{-2.5em}
\end{wrapfigure}

\paragraph{Regulatory activity features and kernel.}
To measure regulatory activity, we use AlphaGenome \citep{avsec2026alphagenome} as an independent sequence-to-function predictor.
For each generated or target sequence, we summarize AlphaGenome-predicted ATAC-seq profiles by their peak heights, yielding a vector of predicted cell-type-specific activity features as shown in \Cref{fig:alpha_genome_feat}.
These feature distributions are matched separately for each cis-regulatory topic.
We apply the preprocessing described in Appendix~\Cref{app_subsec:dna} and use the energy-distance kernel for both training and evaluation.

\begin{figure}[b] 
\centering 
\includegraphics[width=\textwidth]{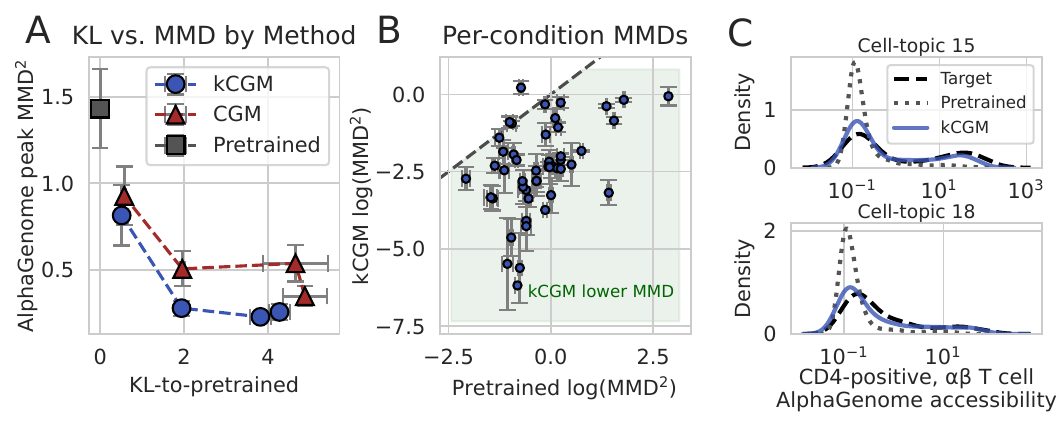} 
\vspace{-2em} 
\caption{ \methodname for conditional regulatory DNA generation. Panel A shows the tradeoff curve between KL to the pretrained model vs. matching the target AlphaGenome-predicted cell-type activity profiles for four values of the KL-regularization weight $\lambda$. Panel B shows MMDs for \methodname with $\lambda = 10^{-2}$ vs. the pretrained model for each of the 47 cell-type conditions. Panel C shows example AlphaGenome peak distributions for CD4-positive $\alpha\beta$ T cells before and after \methodname. } 
\vspace{-1.0em}
\label{fig:dna_results}
\end{figure}

\paragraph{Results.}
We sweep four regularization strengths $\lambda$ for \methodname\ and \CGM, using only training-set AlphaGenome features during finetuning to avoid leakage.
Each finetuning update samples a topic condition and compares generated sequences for that condition to the corresponding topic-specific target feature distribution.
In this experiment, \CGM targets only the mean activity profile for each condition, whereas \methodname\ matches the full activity-feature distribution.
\methodname\ achieves a better KL--MMD tradeoff than \CGM, as measured against held-out target feature distributions for the same topics (\Cref{fig:dna_results}A).
At the per-condition level, \methodname\ improves activity-profile matching for most conditioning topics (\Cref{fig:dna_results}B).
Representative AlphaGenome output tracks show that the pretrained model generates under-active sequences, while \methodname\ shifts the predicted activity distributions toward the topic-specific target (\Cref{fig:dna_results}C).
This shows that \methodname can calibrate a conditional discrete diffusion model using a black-box sequence-to-function predictor.

\section{Discussion and conclusion}
\methodname effectively aligns diffusion and autoregressive models to sample-defined feature distributions.
\methodname works without requiring differentiable features or target samples in the original sample space and maintains or even improves the sample quality of the pretrained model.
Across experiments, the same objective supports hand-designed, learned-embedding, and predictor-derived  features, suggesting that the feature map can be chosen to match the scientific question.
Our results indicate an alternative to direct finetuning when the target dataset is small and only certain aspects of the target distribution are important to preserve.

We provide some practical guidance for picking \methodname hyperparameters in Appendix~\Cref{app:usage_recommendations}.
Our empirical recommendations should be interpreted alongside some technical limitations.
For instance, MMD may be ineffective for very high-dimensional features.
This can be seen in a hypothesis-testing context, where the two-sample statistic we rely on for our objective (\Cref{eq:sample_MMD}) has power that drops as a polynomial function of the feature dimension \citep{ramdas2015decreasing}.
In practice, however, kernel and feature choice, use of PCA, and pretrained-model support may matter more than ambient dimension.
Additionally, \methodname does not immediately apply to generative models without tractable sampling-trajectory log-probabilities, such as flow matching models.
Both limitations may be directions for future work.

In our experiments, the target features are constructed by applying user-chosen feature functions to full target samples.
A natural next step is to use the same framework in settings where experiments provide distributional constraints without full samples in the original generative space.
Such constraints arise in many scientific measurements, including cryo-EM for protein structure \citep{shoemaker2018x}, correlated chemical probing for RNA structure \citep{mustoe2023single}, and proteomics-informed calibration of immunofluorescence image generators.
This would extend feature-distribution matching into a general interface between generative models and partial experimental measurements.

\begin{ack}  
We thank Henry Smith for helpful discussions, help setting up Genie 2 experiments, and providing manuscript feedback; Anshul Kundaje for inspiration for the enhancer DNA generation experiment; Arthur Deng for manuscript feedback; and Shiye Su for discussions about applications of \methodname.

NLD is supported by the Stanford Graduate Fellowship (SGF).
\end{ack}

\bibliographystyle{plainnat}
\bibliography{main}

\newpage
\appendix

\makeatletter
\renewcommand\part{\@ifstar\@spart\@appendixpart}
\newcommand\@appendixpart[1]{%
  \refstepcounter{part}%
  \addcontentsline{toc}{part}{#1}%
  {\centering\Huge\bfseries #1\par}%
  \vspace{1em}%
}
\makeatother

\part{Appendix}
\etocsettocstyle{\section*{Appendix contents}}{}
\localtableofcontents
\newpage

\section{Derivations}\label{app_sec:kCGM_derivations}

\subsection{\CGM-relax as a special case of \methodname}\label{app_subsec:kCGM_CGM}

We show the connection between \methodname and \CGM-relax using the dot-product kernel.
With this kernel, \methodname does not compare full feature distributions and reduces exactly to matching mean feature vectors.
Thus, \CGM-relax can be viewed as the dot-product kernel special case of \methodname, while other kernels extend the objective to capture richer aspects of the target feature distribution.

We first state the population-level relationship.

\begin{proposition}
Let $Q$ be a target distribution on feature space, and let $\mby = \mbh(\mbx)$ for $\mbx \sim p_\theta$, with distribution $P_{\mby}$.
If $k(\mbu,\mbv) = \mbu^\top \mbv$ is the dot-product kernel, then
\[
    \operatorname{MMD}_k^2(P_{\mby}, Q)
    =
    \left\|
        \E[\mby] - \E_{\ytarg \sim Q}[\ytarg]
    \right\|_2^2.
\]
Equivalently, \methodname with the dot-product kernel matches only the mean feature vector of the target distribution.
\end{proposition}

\begin{proof}
By the population definition of squared MMD,
\[
    \operatorname{MMD}_k^2(P_{\mby}, Q)
    =
    \E[\mby^\top \mby']
    - 2\E[\mby^\top \ytarg]
    + \E[(\ytarg)^\top (\ytarg)'],
\]
where $\mby,\mby' \overset{\mathrm{i.i.d.}}{\sim} P_{\mby}$ and $\ytarg,(\ytarg)' \overset{\mathrm{i.i.d.}}{\sim} Q$ are independent.
Using independence,
\[
    \E[\mby^\top \mby']
    =
    \E[\mby]^\top \E[\mby],
    \quad
    \E[\mby^\top \ytarg]
    =
    \E[\mby]^\top \E[\ytarg],
    \quad
    \E[(\ytarg)^\top (\ytarg)']
    =
    \E[\ytarg]^\top \E[\ytarg].
\]
Substituting gives
\[
    \operatorname{MMD}_k^2(P_{\mby}, Q)
    =
    \E[\mby]^\top \E[\mby]
    - 2\E[\mby]^\top \E[\ytarg]
    + \E[\ytarg]^\top \E[\ytarg]
    =
    \left\|
        \E[\mby] - \E[\ytarg]
    \right\|_2^2.
\]
\end{proof}

As an immediate corollary, if the target distribution $Q$ is concentrated at a single value $\mbh^\ast$, then dot-product \methodname reduces to matching $\E[\mbh(\mbx)]$ to $\mbh^\ast$, which is exactly the population objective used by \CGM-relax:
\[
    \mathcal{L}_\theta[\mbh^\ast]
    =
    \left\|
        \E_{\mbx \sim p_{\theta}}[\mbh(\mbx)] - \mbh^\ast
    \right\|_2^2
    +
    \lambda \operatorname{D_{KL}}(p_\theta || \ppre).
\]

At the stochastic-optimization level, the connection is even tighter.
Let $\hat{\mathcal L}^{\mathrm{viol}}$ denote the finite-sample \CGM-relax violation objective used to match the empirical target mean.
When the target feature distribution is represented by empirical samples, dot-product \methodname depends on those samples only through their empirical mean, up to a target-only constant.
As a result, dot-product \methodname recovers the same score-function optimization objective as \CGM-relax.

\begin{proposition}
Let $\ytarg_{1:N}$ be target feature samples with empirical mean
\[
    \bar{\mby}^{\mathrm{targ}}
    :=
    \frac{1}{N}\sum_{n=1}^N \ytarg_n.
\]
Let $\mby_{1:M}$ be model feature samples, where $\mby_m = \mbh(\mbx_m)$ for $\mbx_m \sim p_\theta$, and take the dot-product kernel
\[
    k(\mbu,\mbv) = \mbu^\top \mbv.
\]
Then the unbiased finite-sample estimator of squared MMD satisfies
\[
    \widehat{\operatorname{MMD}}_k^2[\mby_{1:M}, \ytarg_{1:N}]
    =
    \frac{1}{M(M-1)}\sum_{m \neq m'} \mby_m^\top \mby_{m'}
    -
    \frac{2}{M}\sum_{m=1}^M \mby_m^\top \bar{\mby}^{\mathrm{targ}}
    + C_{\mathrm{targ}},
\]
where
\[
    C_{\mathrm{targ}}
    :=
    \frac{1}{N(N-1)}\sum_{n \neq n'} (\ytarg_n)^\top \ytarg_{n'}
\]
depends only on the target samples.

Equivalently,
\[
    \widehat{\operatorname{MMD}}_k^2[\mby_{1:M}, \ytarg_{1:N}]
    =
    \hat{\mathcal L}^{\mathrm{viol}}(\mby_{1:M};\bar{\mby}^{\mathrm{targ}})
    + C'_{\mathrm{targ}},
\]
where
\[
    \hat{\mathcal L}^{\mathrm{viol}}(\mby_{1:M};\bar{\mby}^{\mathrm{targ}})
    :=
    \frac{1}{M(M-1)}\sum_{m \neq m'}
    (\mby_m-\bar{\mby}^{\mathrm{targ}})^\top
    (\mby_{m'}-\bar{\mby}^{\mathrm{targ}}),
\]
and $C'_{\mathrm{targ}}$ is constant with respect to $\theta$.
\end{proposition}

\begin{proof}
For the dot-product kernel,
\[
\begin{aligned}
    \widehat{\operatorname{MMD}}_k^2[\mby_{1:M}, \ytarg_{1:N}]
    =
    &\frac{1}{M(M-1)}
    \sum_{m \neq m'} \mby_m^\top \mby_{m'}
    -
    \frac{2}{MN}
    \sum_{m=1}^M \sum_{n=1}^N
    \mby_m^\top \ytarg_n \\
    &+
    \frac{1}{N(N-1)}
    \sum_{n \neq n'} (\ytarg_n)^\top \ytarg_{n'}.
\end{aligned}
\]
Rearranging the middle term gives
\[
    \frac{2}{MN}
    \sum_{m=1}^M \sum_{n=1}^N
    \mby_m^\top \ytarg_n
    =
    \frac{2}{M}
    \sum_{m=1}^M
    \mby_m^\top \bar{\mby}^{\mathrm{targ}},
\]
so
\[
    \widehat{\operatorname{MMD}}_k^2[\mby_{1:M}, \ytarg_{1:N}]
    =
    \frac{1}{M(M-1)}
    \sum_{m \neq m'} \mby_m^\top \mby_{m'}
    -
    \frac{2}{M}
    \sum_{m=1}^M
    \mby_m^\top \bar{\mby}^{\mathrm{targ}}
    +
    C_{\mathrm{targ}}.
\]
On the other hand,
\[
    \hat{\mathcal L}^{\mathrm{viol}}(\mby_{1:M};\bar{\mby}^{\mathrm{targ}})
    =
    \frac{1}{M(M-1)}\sum_{m \neq m'}
    (\mby_m-\bar{\mby}^{\mathrm{targ}})^\top
    (\mby_{m'}-\bar{\mby}^{\mathrm{targ}}),
\]
which expands to
\[
    \frac{1}{M(M-1)}
    \sum_{m \neq m'} \mby_m^\top \mby_{m'}
    -
    \frac{2}{M}
    \sum_{m=1}^M
    \mby_m^\top \bar{\mby}^{\mathrm{targ}}
    +
    \|\bar{\mby}^{\mathrm{targ}}\|_2^2.
\]
Therefore
\[
    \widehat{\operatorname{MMD}}_k^2[\mby_{1:M}, \ytarg_{1:N}]
    =
    \hat{\mathcal L}^{\mathrm{viol}}(\mby_{1:M};\bar{\mby}^{\mathrm{targ}})
    +
    \left(
        C_{\mathrm{targ}}
        -
        \|\bar{\mby}^{\mathrm{targ}}\|_2^2
    \right),
\]
and the difference depends only on the target samples.
\end{proof}

As a consequence, the expected dot-product \methodname objective differs from the \CGM-relax mean-matching objective with target mean $\bar{\mby}^{\mathrm{targ}}$ only by a constant independent of $\theta$.
Therefore the score-function gradient of dot-product \methodname is identical to that of \CGM-relax with target mean $\bar{\mby}^{\mathrm{targ}}$.

\subsection{Kernel discrepancy coefficients for score-function gradients}
\label{app_subsec:kernel_coeffs}

As a reminder, $\mbx \sim p_\theta$ denotes a full sampling trajectory.
The feature map $\mbh(\mbx)$ is understood to act on the terminal generated sample of that trajectory.
Thus $\logp_\theta(\mbx)$ denotes the sampling-trajectory log-probability.

We now derive the per-sample kernel discrepancy coefficients used in Algorithm~\ref{alg:kcgm}.
Since \methodname\ is designed to handle non-differentiable feature functions $\mbh(\cdot)$ and non-differentiable sampling procedures, we derive these coefficients directly from the score-function gradient of the expected squared MMD objective
\begin{equation}
    \label{eq:kernel_objective_expectation}
    \mathcal{L}^{\mathrm{k}}_\theta[\ytarg_{1:N}]
    :=
    \E_{\mbx_{1:M} \overset{\mathrm{i.i.d.}}{\sim} p_\theta}
    \left[
        \widehat{\operatorname{MMD}}^2_k[\mby_{1:M}, \ytarg_{1:N}]
    \right],
    \qquad
    \mby_m := \mbh(\mbx_m).
\end{equation}
The target--target term in \Cref{eq:sample_MMD} is constant with respect to $\theta$, so it can be ignored when taking gradients. Thus
\begin{align}
    \nabla_\theta \mathcal{L}^{\mathrm{k}}_\theta[\ytarg_{1:N}]
    &=
    \nabla_\theta
    \E_{\mbx_{1:M} \overset{\mathrm{i.i.d.}}{\sim} p_\theta}
    \underbrace{\left[
        \frac{1}{M(M-1)}
        \sum_{m\neq m'}^M
        k(\mby_m, \mby_{m'})
        -
        \frac{2}{MN}
        \sum_{m=1}^M\sum_{n=1}^N
        k(\mby_m, \ytarg_n)
    \right]}_{[\text{kernel terms}]}
    \nonumber\\
    &=
    \E_{\mbx_{1:M} \overset{\mathrm{i.i.d.}}{\sim} p_\theta}
    \Biggl[
        [\text{kernel terms}]
        \sum_{m=1}^M \nabla_\theta \logp_\theta(\mbx_m)
    \Biggr].
    \label{eq:score_fn_joint}
\end{align}

We first derive the contribution of the model--model term. Define
\[
    \phi(\mbx,\mbx') := k(\mbh(\mbx), \mbh(\mbx')).
\]
Then
\begin{align}
    &\nabla_\theta
    \E_{\mbx,\mbx' \overset{\mathrm{i.i.d.}}{\sim} p_\theta}
    \bigl[\phi(\mbx,\mbx')\bigr]
    \nonumber\\
    &=
    \nabla_\theta
    \int \phi(\mbx,\mbx')\, p_\theta(\mbx)\, p_\theta(\mbx') \, d\mbx\, d\mbx'
    \nonumber\\
    &=
    \int \phi(\mbx,\mbx')
    \Bigl(
        \nabla_\theta p_\theta(\mbx)\, p_\theta(\mbx')
        +
        p_\theta(\mbx)\, \nabla_\theta p_\theta(\mbx')
    \Bigr)
    d\mbx\, d\mbx'
    \nonumber\\
    &=
    \E_{\mbx,\mbx' \overset{\mathrm{i.i.d.}}{\sim} p_\theta}
    \left[
        \phi(\mbx,\mbx')
        \bigl(
            \nabla_\theta \logp_\theta(\mbx)
            +
            \nabla_\theta \logp_\theta(\mbx')
        \bigr)
    \right].
    \label{eq:pairwise_product_rule}
\end{align}
Since $\phi(\mbx,\mbx') = \phi(\mbx',\mbx)$ by symmetry of the kernel and $(\mbx,\mbx')$ are exchangeable, the two expectation terms are equal, so
\begin{align}
    \nabla_\theta
    \E_{\mbx,\mbx' \overset{\mathrm{i.i.d.}}{\sim} p_\theta}
    \bigl[k(\mbh(\mbx), \mbh(\mbx'))\bigr]
    =
    2\,
    \E_{\mbx,\mbx' \overset{\mathrm{i.i.d.}}{\sim} p_\theta}
    \left[
        k(\mbh(\mbx), \mbh(\mbx'))
        \nabla_\theta \logp_\theta(\mbx)
    \right].
    \label{eq:pairwise_factor_two}
\end{align}
Applying \Cref{eq:pairwise_factor_two} to the model--model term in \Cref{eq:score_fn_joint}, and the ordinary score-function trick to the model--target term, gives
\begin{equation}
\resizebox{\textwidth}{!}{$\displaystyle
\begin{aligned}
    \nabla_\theta \mathcal{L}^{\mathrm{k}}_\theta[\ytarg_{1:N}]
    =
    \E_{\mbx_{1:M} \overset{\mathrm{i.i.d.}}{\sim} p_\theta}
    \Biggl[
        \sum_{m=1}^M
        \Biggl(
            \frac{2}{M(M-1)}
            \sum_{m' \neq m} k(\mby_m, \mby_{m'})
            -
            \frac{2}{MN}
            \sum_{n=1}^N k(\mby_m, \ytarg_n)
        \Biggr)
        \nabla_\theta \logp_\theta(\mbx_m)
    \Biggr].
\end{aligned}
$}
\end{equation}

Define
\begin{align}
    s_m &:= \frac{1}{M-1}\sum_{m' \neq m} k(\mby_m, \mby_{m'}), \\
    r_m &:= \frac{1}{N}\sum_{n=1}^N k(\mby_m, \ytarg_n), \\
    g_m &:= 2\left(s_m - r_m\right), \\
    c_m^{\mathrm{k}} &:= \frac{1}{M} g_m.
\end{align}
Then
\begin{equation}
    \label{eq:kernel_grad_ck}
    \nabla_\theta \mathcal{L}^{\mathrm{k}}_\theta[\ytarg_{1:N}]
    =
    \E_{\mbx_{1:M} \overset{\mathrm{i.i.d.}}{\sim} p_\theta}
    \left[
        \sum_{m=1}^M c_m^{\mathrm{k}} \nabla_\theta \logp_\theta(\mbx_m)
    \right].
\end{equation}
Thus $c_m^{\mathrm{k}}$ are exactly the kernel discrepancy coefficients used in Algorithm~\ref{alg:kcgm}.

\subsection{Unbiased surrogate gradients}
\label{app_subsec:unbiased_grads}

Under standard regularity conditions, we may exchange gradient and expectation in the score-function expressions above.
We now show that the stop-gradient surrogate used in Algorithm~\ref{alg:kcgm} gives an unbiased estimator of the gradient of the regularized objective.

\begin{proposition}
The stop-gradient surrogate gradient used by \methodname\ is an unbiased estimator of
\[
    \nabla_\theta \mathcal{L}_\theta[\ytarg_{1:N}]
    =
    \nabla_\theta \mathcal{L}^{\mathrm{k}}_\theta[\ytarg_{1:N}]
    +
    \lambda \nabla_\theta \operatorname{D_{KL}}(p_\theta || \ppre).
\]
\end{proposition}

\begin{proof}
From \Cref{app_subsec:kernel_coeffs}, we have
\[
    \nabla_\theta \mathcal{L}^{\mathrm{k}}_\theta[\ytarg_{1:N}]
    =
    \E_{\mbx_{1:M} \overset{\mathrm{i.i.d.}}{\sim} p_\theta}
    \left[
        \sum_{m=1}^M c_m^{\mathrm{k}} \nabla_\theta \logp_\theta(\mbx_m)
    \right].
\]
For the KL term, \citet{smith2025calibrating} derive the analogous score-function form
\[
    \nabla_\theta \operatorname{D_{KL}}(p_\theta || \ppre)
    =
    \E_{\mbx_{1:M} \overset{\mathrm{i.i.d.}}{\sim} p_\theta}
    \left[
        \sum_{m=1}^M c_m^{\mathrm{KL}} \nabla_\theta \logp_\theta(\mbx_m)
    \right],
\]
where
\[
    l_m := \logp_\theta(\mbx_m) - \logp_{\thetapre}(\mbx_m),
    \qquad
    c_m^{\mathrm{KL}} := \frac{1}{M} l_m.
\]
Combining the two terms gives
\[
    \nabla_\theta \mathcal{L}_\theta[\ytarg_{1:N}]
    =
    \E_{\mbx_{1:M} \overset{\mathrm{i.i.d.}}{\sim} p_\theta}
    \left[
        \sum_{m=1}^M
        \bigl(\lambda c_m^{\mathrm{KL}} + c_m^{\mathrm{k}}\bigr)
        \nabla_\theta \logp_\theta(\mbx_m)
    \right].
\]
Now define
\[
    c_m := \lambda c_m^{\mathrm{KL}} + c_m^{\mathrm{k}},
    \qquad
    \widehat{\mathcal L}^{\mathrm{\methodname\text{-}surrogate}}
    :=
    \sum_{m=1}^M c_m \logp_\theta(\mbx_m),
\]
where the coefficients $c_m$ are computed from samples drawn from $p_{\texttt{stop-grad}(\theta)}$ and are therefore treated as constants with respect to the gradient step.
Hence
\[
    \nabla_\theta \widehat{\mathcal L}^{\mathrm{\methodname\text{-}surrogate}}
    =
    \sum_{m=1}^M c_m \nabla_\theta \logp_\theta(\mbx_m).
\]
Taking expectations over $\mbx_{1:M} \overset{\mathrm{i.i.d.}}{\sim} p_\theta$ yields
\[
    \E_{\mbx_{1:M} \overset{\mathrm{i.i.d.}}{\sim} p_\theta}
    \left[
        \nabla_\theta \widehat{\mathcal L}^{\mathrm{\methodname\text{-}surrogate}}
    \right]
    =
    \nabla_\theta \mathcal{L}_\theta[\ytarg_{1:N}].
\]
Thus the surrogate gradient is unbiased.
\end{proof}

\subsection{Leave-one-out baselines}\label{app_subsec:loo}

To reduce variance, Algorithm~\ref{alg:kcgm} applies leave-one-out centering to the kernel coefficients.
Intuitively, the baseline for sample $m$ should approximate the average coefficient that would have been obtained had $\mbx_m$ not appeared in the batch.
This does not change the expectation provided the baseline subtracted from the coefficient for sample $m$ is a function only of $\mbx_{-m}$, since for any $b_m(\mbx_{-m},\ytarg_{1:N})$ that does not depend on $\mbx_m$,
\[
    \E_{\mbx_{1:M} \sim p_\theta}
    \left[
        b_m(\mbx_{-m},\ytarg_{1:N})
        \nabla_\theta \logp_\theta(\mbx_m)
    \right]
    =
    0.
\]
Indeed,
\[
\begin{aligned}
    &\E_{\mbx_{1:M} \sim p_\theta}
    \left[
        b_m(\mbx_{-m},\ytarg_{1:N})
        \nabla_\theta \logp_\theta(\mbx_m)
    \right]
    \\
    &\qquad =
    \E_{\mbx_{-m} \sim p_\theta}
    \left[
        b_m(\mbx_{-m},\ytarg_{1:N})
        \E_{\mbx_m \sim p_\theta}
        \left[
            \nabla_\theta \logp_\theta(\mbx_m)
        \right]
    \right]
    =
    0.
\end{aligned}
\]

For the kernel term, we therefore define the leave-one-out baseline for sample $m$ by recomputing the coefficients of the remaining samples on the reduced batch $\mbx_{-m}$.
Concretely, for $j \neq m$, define
\[
    s_j^{(-m)}
    :=
    \frac{1}{M-2}\sum_{\ell \neq j,m} k(\mby_j,\mby_\ell),
    \qquad
    r_j^{(-m)} := r_j,
\]
and
\[
    g_j^{(-m)} := 2\bigl(s_j^{(-m)} - r_j\bigr),
    \qquad
    c_j^{\mathrm{k},(-m)} := \frac{1}{M} g_j^{(-m)}.
\]
We then set
\[
    b_m^{\mathrm{k}}(\mbx_{-m},\ytarg_{1:N})
    :=
    \frac{1}{M-1}\sum_{j \neq m} c_j^{\mathrm{k},(-m)},
\]
and use the centered coefficient
\[
    c_m^{\mathrm{k,LOO}}
    :=
    c_m^{\mathrm{k}} - b_m^{\mathrm{k}}(\mbx_{-m},\ytarg_{1:N}).
\]
By construction, $b_m^{\mathrm{k}}(\mbx_{-m},\ytarg_{1:N})$ depends only on $\mbx_{-m}$ and $\ytarg_{1:N}$, so replacing $c_m^{\mathrm{k}}$ with $c_m^{\mathrm{k,LOO}}$ preserves unbiasedness.

Naively computing $b_m^{\mathrm{k}}(\mbx_{-m},\ytarg_{1:N})$ for every $m$ would require recomputing the model--model coefficients separately for each leave-one-out batch.
However, all leave-one-out coefficients can be obtained efficiently from the same model feature kernel matrix $K_{\mby\mby}$, where $(K_{\mby\mby})_{ij} = k(\mby_i,\mby_j)$.
In particular, for $j \neq m$,
\begin{align}
    s_j^{(-m)}
    &:=
    \frac{1}{M-2}\sum_{\ell \neq j,m} k(\mby_j,\mby_\ell)
    \\
    &=
    \frac{\sum_{\ell \neq j} k(\mby_j,\mby_\ell) - k(\mby_j,\mby_m)}{M-2}
    \\
    &=
    \frac{(M-1)s_j - k(\mby_j,\mby_m)}{M-2},
    \label{eq:kernel_loo_self}
    \\
    g_j^{(-m)}
    &:=
    2\bigl(s_j^{(-m)} - r_j\bigr),
    \qquad
    c_j^{\mathrm{k},(-m)} := \frac{1}{M} g_j^{(-m)},
    \label{eq:kernel_loo_coeff}
    \\
    b_m^{\mathrm{k}}(\mbx_{-m},\ytarg_{1:N})
    &:=
    \frac{1}{M-1}\sum_{j \neq m} c_j^{\mathrm{k},(-m)},
    \qquad
    c_m^{\mathrm{k,LOO}}
    :=
    c_m^{\mathrm{k}} - b_m^{\mathrm{k}}(\mbx_{-m},\ytarg_{1:N}).
    \label{eq:kernel_loo_baseline}
\end{align}
Thus all $b_m^{\mathrm{k}}(\mbx_{-m},\ytarg_{1:N})$ can be computed in $O(M^2)$ time from the row sums of $K_{\mby\mby}$, without forming additional kernel matrices.
Algorithm~\ref{alg:kcgm} uses the centered coefficients $c_m^{\mathrm{k,LOO}}$ from \Cref{eq:kernel_loo_baseline}.

Algorithm~\ref{alg:kcgm} also applies the ordinary leave-one-out baseline to the KL coefficients, exactly as in \citet{smith2025calibrating}.
The final coefficients used by \methodname\ therefore remain unbiased while typically having lower variance.

\subsection{Behavior of self-repulsion weight less than one for categorical target features}\label{app_subsec:cat_self_wt}

We now show explicitly how the self-repulsion weight $\alpha < 1$ changes the target of the population objective in a simple categorical setting. Suppose the feature space is $\{1,\ldots,C\}$, the target feature distribution is
\[
    Q = (Q_1,\ldots,Q_C) \in \Delta_C,
\]
and the model-induced feature distribution is
\[
    P = (P_1,\ldots,P_C) \in \Delta_C.
\]
Equivalently, $Q_i$ is the probability that a target feature equals category $i$, and $P_i$ is the probability that a model sample has feature category $i$. We represent category $i$ by the one-hot vector $\mbe_i \in \mathbb{R}^C$ and use the dot-product kernel
\[
    k(\mbe_i,\mbe_j) = \mbe_i^\top \mbe_j = \mathbbm{1}\{i=j\}.
\]
The population version of the $\alpha$-weighted MMD objective is then
\begin{align}
    \mathcal{L}_\alpha(P;Q)
    &=
    \alpha \E_{\mbz,\mbz' \iid \sim  P}
        \left[k(\mbz,\mbz')\right]
    -
    2 \E_{\mbz \sim P, \mby \sim Q}
        \left[k(\mbz,\mby)\right]
    +
    \E_{\mby,\mby' \iid \sim Q}
        \left[k(\mby,\mby')\right] \\
    &=
    \alpha \sum_{i=1}^C P_i^2
    -
    2 \sum_{i=1}^C P_i Q_i
    +
    \sum_{i=1}^C Q_i^2 .
\end{align}
Completing the square gives
\begin{align}
    \mathcal{L}_\alpha(P;Q)
    &=
    \alpha
    \sum_{i=1}^C
    \left(P_i - \frac{Q_i}{\alpha}\right)^2
    +
    \left(1-\frac{1}{\alpha}\right)
    \sum_{i=1}^C Q_i^2 .
\end{align}
The second term is constant with respect to $P$, and $\alpha>0$, so the optimizer over categorical distributions is
\[
    P^\star
    =
    \argmin_{P \in \Delta_C}
    \left\|
        P - \frac{Q}{\alpha}
    \right\|_2^2
\]
This problem is known as simplex projection and has the closed form solution \citep{duchi2008efficient}
\[
    P_i^\star
    =
    \left(
        \frac{Q_i}{\alpha} - \nu
    \right)_+,
\]
where $(u)_+ := \max\{u,0\}$ and $\nu$ is chosen so that
$\sum_i P_i^\star = 1$. Equivalently, defining
$\tau_\alpha := \alpha \nu$, we can write
\[
    P_i^\star
    =
    \frac{(Q_i-\tau_\alpha)_+}{\alpha},
    \qquad
    \sum_{i=1}^C (Q_i-\tau_\alpha)_+ = \alpha .
\]
When $\alpha=1$, we may take $\tau_\alpha=0$, giving $P^\star=Q$. When $\alpha<1$, the optimizer subtracts a common threshold from the target probabilities, sets categories below that threshold to zero, and rescales the remaining mass by $1/\alpha$. Therefore, in this categorical setting, reducing $\alpha$ does not target the original distribution $Q$ exactly. Instead, it produces a thresholded sharpening of $Q$, concentrating probability mass on higher-probability categories.

\section{Practical usage recommendations}\label{app:usage_recommendations}
Our experiments suggest some practical guidance.
If the target features are continuous, then the energy-distance kernel will likely work well and has no extra hyperparameters.
We recommend doing a small sweep over the regularization weight $\lambda = \{10^{-1}, 10^{-2}, 10^{-3}\}$ while keeping self-repulsion $\alpha = 1$.
Lower $\lambda$ can make optimization unstable.
Trying $\alpha = \{0.25, 0.5\}$ can additionally improve distribution-matching performance as seen in our protein structure experiment.
For discrete-valued or otherwise non-Euclidean kernels, use domain-specific kernels such as the Tanimoto kernel.

We also recommend two scaling tricks to help hyperparameters like learning rate transfer across $\lambda$ and kernel choices.
First, rescale the MMD and KL-regularization terms so that their ratio remains $\lambda$, but the overall scale of the objective stays roughly constant via \Cref{eq:reweighted_objective}.
Second, scale the kernel values so that the initial squared MMD between pretrained generations and target features is fixed.

\section{Additional experimental details}\label{app:additional_details}
In this section, we provide additional information about each of our experiments, including how we compute log-probabilities and how we define the feature function $\mbh$ and the feature target samples $\ytarg_{1:N}$.

All experiments use Adam with momentum hyperparameters $\beta = (0.9, 0.999)$ and either a cosine or linear decay learning rate schedule.
We train all models on either a single H100 or L40S GPU with at most 4 CPUs.
Genie 2 experiments take 30 hours per finetuning run on L40S and G2PT runs take one hour.

\subsection{Protein secondary structure distribution matching}

\paragraph{Features and kernels.}
Following \citet{smith2025calibrating}, we use the Biotite package (\url{https://www.biotite-python.org/latest/index.html}) to annotate secondary structure.

For \methodname, we use the proportion of alpha-helix and beta-strand residues to compute a two-vector of features.
These are compared using the energy-distance kernel (\Cref{eq:energy_kernel}).

For \CGM, we use the bivariate quantile features implemented by \citet{smith2025calibrating}.
The quantile bins are initialized using the empirical quantiles from the CATH target dataset.

\paragraph{Sampling and log-probabilities.}
All optimization hyperparameters and score-function gradient computations were taken from the code of  \citet{smith2025calibrating} (\url{https://github.com/smithhenryd/cgm/blob/main/genie2/calibrate_genie2.py}).
In particular, see their appendix section on diffusion models for a more detailed treatment than follows.

For completeness, we briefly outline the approach to sampling and computing score-function gradients.
We use Euler-Maruyama sampling with 100 sampling steps.
Each sampling step $i$ consists of estimating a drift $\mbb_\theta(\mbx_i, i / 100)$ using the model, then sampling the next time $\mbx_{i+1}$ from a Gaussian with mean and variance determined by the drift and noise schedule.
Therefore the log-probability of the sampling trajectory can be decomposed into a sum of log probabilities of Gaussians with means determined by the model weights, enabling tractable gradient computations.

The decomposition over sampling steps also enables accumulating gradients over time steps, avoiding holding the entire computational graph in memory, making finetuning tractable on one GPU even with $100$ sampling steps.

\paragraph{Additional training and evaluation details.}
Following \citet{smith2025calibrating}, we train for 100 epochs with batch size 64 and a learning rate that decays linearly from $10^{-5}$ to $10^{-7}$ and sampling noise level $\sigma = 0.5$ over 100 sampling steps.
Each batch is split into 20 temporal sampling chunks by 32 batch dimension chunks.

For the symmetrized KL evaluation metric, we compare two empirical bivariate feature distributions using kernel density estimates. 
Each sample was represented by two
features taking values in $[0,1]^2$, such as secondary-structure proportions.
Given samples from two distributions, we fit two Gaussian kernel density estimates using
\texttt{scipy.stats.gaussian\_kde} with default bandwidth selection rule, Scott's rule.
We then approximated the KL divergence by deterministic numerical integration
on a uniform $50\times50$ grid over $[0,1]^2$.

\subsection{Antibiotics feature distribution matching}\label{subsec:abx_app}

\paragraph{Sampling and log-probabilities.}
Sampling and log-probabilities for the G2PT model proceed in the standard way for autoregressive models.
That is, tokens are sequentially sampled according to $\mbx_i \sim p_\theta(\mbx_i \mid \mbx_{{<}i})$, where $\mbx_{{<}i}$ indicates the previous $i - 1$ tokens in the sequence.
Autoregressive log-probabilities are the same as used for training, 
\[
    \logp_\theta(\mbx) = \sum_{i=1}^T \log p_\theta(\mbx_i \mid \mbx_{{<}i}),
\]
where $T$ is the sequence length including an end-of-sequence token.
This can be computed with one forward pass since G2PT is an autoregressive transformer.

\paragraph{Kernel rescaling.}
In order to make it easier to share hyperparameters across all three target features, we rescale each kernel during training so that the MMD$^2$ is $10$ for the pretrained model.
We do this by sampling five batches from the pretrained model, computing average MMD$^2$ against the target dataset for each batch $\overline{\operatorname{MMD}^2}[\mby_{1:M},\ytarg_{1:N}]$, then defining 
\[
    k'(\mbx, \mby) = \frac{10 k(\mbx, \mby)}{\overline{\operatorname{MMD}^2}[\mby_{1:M}, \ytarg_{1:N}]}.
\]

\paragraph{Features and kernels.}

For all features, chemically invalid generations are mapped to the zero vector.

For fingerprint embeddings, we use RDKit \citep{greg_landrum_rdkit} to compute Morgan fingerprints with 256 bits and radius two and use the Tanimoto kernel (\Cref{eq:tanimoto_kernel}).

For scaffold embeddings, we use RDKit to extract Murcko generic scaffolds, then convert those to SMILES strings.
Those SMILES are then passed to \texttt{fcd\_torch} (\url{https://github.com/insilicomedicine/fcd_torch}) in order to get ChemNet embeddings.
Chemically invalid generations are mapped to the zero vector.
Finally, we use PCA to reduce the embedding dimension to 64.
The PCs are fit on the target antibiotics dataset embeddings.

For our molecular descriptor features, we use RDKit to compute molecular weight, logP, topological polar surface area, number of rotatable bonds, QED, synthetic accessibility score, fraction of sp$^3$ carbons, numbers of hydrogen-bond donors and acceptors, number of aromatic rings, and heavy atom count.
We divide each by its empirical standard deviation in the antibiotics target set so they are all roughly equally weighted, but do not zero-center them.

In \Cref{fig:murcko_match} we compare the distance between embeddings that share a generic Murcko scaffold against embeddings with different scaffolds.
ChemNet embeddings of Murcko scaffolds (with and without PCA to 64 components) compared via Euclidean distance both have distance zero when scaffolds are shared by construction, but also show a wide range of distances for distinct scaffolds.
Morgan fingerprints with Tanimoto distance (one minus Tanimoto similarity) do not separate by scaffold identity.

\begin{figure}[ht]
  \vskip 0.2in
  \begin{center}
    \centering
    \includegraphics[width=\textwidth]{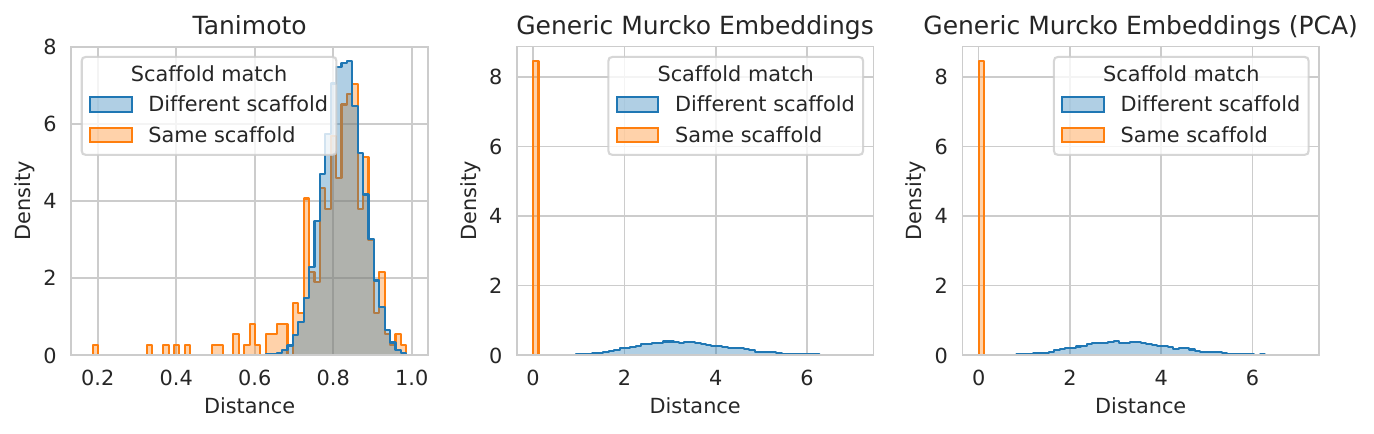}
    \caption{
Comparison of molecular similarity measures for identifying shared generic Murcko scaffolds. Histograms show pairwise distances for molecule pairs with the same generic scaffold and with different scaffolds. Morgan fingerprint Tanimoto distance only weakly separates these groups, while distances in the generic Murcko embedding space separate them much more clearly. PCA preserves this separation, suggesting that dimensionality reduction has little effect on scaffold identification.
    }
    \label{fig:murcko_match}
  \end{center}
\end{figure}

\paragraph{Additional training and evaluation details.}
For each finetuning method, including the \CGM and direct finetuning baselines, we train for 500 epochs with batch size 192 and evaluate with 9,600 samples.
We use a cosine-annealing-with-warm-up learning rate schedule with 25 warm-up epochs and final learning rate of $0.01$ times the initial.
We tried learning rates $5\times 10^{-6}, 1\times10^{-5}, 5\times 10^{-5}$ for each method and took the learning rate for each with the lowest distribution-matching loss for the Morgan fingerprint objective averaged over runs with $\lambda = 10$ and $\lambda = 0.01$ to ensure the learning rate worked for high and low regularization weights.
Towards this end, we found it helpful to normalize the regularization and MMD objectives' weightings such that their ratio was preserved while keeping the absolute scale of the loss roughly conserved across regularization weights.
This resulted in the reweighted objective
\begin{equation}\label{eq:reweighted_objective}
    \mathcal{L}_\theta[k, \ytarg_{1:N}] = \frac{1}{1 + \lambda}\E_{\mbx_{1:M} \overset{i.i.d.}{\sim} p_{\theta}}[\widehat{\operatorname{MMD}}^2_k[\mby_{1:M}, \ytarg_{1:N}]] + \frac{\lambda}{1 + \lambda} \operatorname{D_{KL}}(p_\theta || \ppre).  
\end{equation}
This resulted in all methods having optimal performance at learning rate $10^{-5}$.

\paragraph{Direct finetuning baseline.}
The direct finetuning baseline uses the full antibiotic molecules as target samples in the original sample space, rather than using only feature-space samples.
Starting from the same pretrained G2PT model as \methodname and \CGM, we finetune by minimizing the standard autoregressive negative log-likelihood of target antibiotic token sequences denoted $\mathcal{D}_{\mathrm{abx}}$,
\[
    \mathcal{L}_{\mathrm{sup}}(\theta)
    =
    - \mathbb{E}_{\mbx \sim \mathcal{D}_{\mathrm{abx}}}
    \left[\logp_\theta(\mbx)\right]
    +
    \lambda D_{\mathrm{KL}}(p_\theta \| p_{\theta_{\mathrm{pre}}}),
\]
where $\logp_\theta(\mbx)$ is the G2PT sequence log-probability.
We sweep the same KL-regularization weights used for \methodname and \CGM so that all methods are compared at matched KL-to-pretrained budgets.
Unlike \methodname and \CGM, this baseline requires full target molecules and trains the model to imitate the small antibiotic set directly.

\paragraph{Antibiotics target dataset.}
We constructed the antibiotic target set used in the G2PT experiments from the set of known antibiotics from \citet{scalia2025deep}.
We process this set to produce a target set of SMILES strings that are unique, chemically valid, and compatible with
the exact molecular representation pipeline used during G2PT finetuning and
calibration. 
To do this, we iterated through the input SMILES records and split any
multi-component entry at the ``.'' separator, so that mixtures or salts contributed one row per molecular fragment rather than being treated as a single composite molecule. 
Each fragment was then parsed with RDKit and canonicalized with
isomeric information removed, so duplicates differing only by SMILES syntax or stereochemical annotation collapsed to a common canonical representation. 
Fragments that could not be parsed by RDKit were discarded.

After canonicalization, each candidate SMILES string was filtered by running it through the same preprocessing steps used by the downstream G2PT pipeline. 
In particular, the molecule was converted into the graph representation used by G2PT, linearized with the breadth-first-search serialization used by the model, and then tokenized with the pretrained GuacaMol G2PT tokenizer. 
A candidate was retained only if all of these steps succeeded, the number of atoms did not exceed the tokenizer's supported indexed-atom vocabulary, and the resulting token sequence fit within the tokenizer's maximum input length. This ensures that every molecule is  processable by the tokenizer used in the experiments. Finally, exact duplicate canonical SMILES were removed.

\subsection{Generating DNA with realistic cell-type-specific activity profiles}\label{app_subsec:dna}

\paragraph{Masked discrete diffusion model pretraining.}
We pretrained a conditional masked discrete diffusion model on DeepMEL2 enhancer
sequences. DNA sequences were represented as categorical tokens over the four bases, with an additional absorbing mask token used by the diffusion process.
The model was trained for 1,000 epochs on the DeepMEL2 training split on a single L40S GPU with batch size
256, learning rate $5\times 10^{-4}$, Adam optimization, and validation-loss checkpointing. 
The run used a one-dimensional convolutional sequence model with 512 hidden channels, dropout 0.1, and one dilation block.

The architecture follows a ChromBPNet-style dilated convolutional design \citep{pampari2025chrombpnet}. Input
tokens are one-hot encoded over the augmented alphabet
$\{\mathrm{A},\mathrm{C},\mathrm{G},\mathrm{T},\mbm\}$ and passed through an
initial convolution where $\mbm$ denotes the mask token. The model then applies a residual stack of one-dimensional
convolutional blocks with dilation pattern $(1,1,4,16,64)$. Each block uses channel-wise layer normalization, a same-padded dilated convolution, ReLU nonlinearity, and dropout. A final $1$D convolution maps the hidden
representation to logits over the four unmasked DNA bases at each position.

The model is conditional on the DeepMEL2 cell-topic annotation vector. The conditioning vector is passed through a small multilayer perceptron consisting
of layer normalization, a linear layer to the model hidden dimension, a ReLU,
and a second linear layer. 
The resulting embedding is added to the sequence hidden representation at the beginning of each dilation block, so the same condition vector can globally modulate predictions at all sequence positions.

Training uses the masked diffusion denoising objective. For each minibatch, a diffusion time $t\in[0,1]$ is sampled using a low-discrepancy grid with a random
offset \citep{kingma2021variational}. 
Given a clean sequence $\mbx_0$, the forward process independently keeps
each token with probability
\[
    \alpha(t) = 1 - \cos\left(\frac{\pi}{2}(1-t)\right)
\]
and replaces it by the mask token $\mbm$ otherwise. The network receives the corrupted
sequence $\mbx_t$ and the condition vector, and is trained to predict the
original base at masked positions. The per-sequence loss is a weighted masked
cross entropy,
\[
    \mathcal{L}(\theta)
    =
    \sum_{j : x_t[j]=\mbm}
    w(t)\,
    \mathrm{CE}\left(
        x_0[j],
        \pi_\theta^{(j)}(\cdot \mid \mbx_t,\mbc)
    \right),
\]
normalized by sequence length. We use the cosine diffusion loss weighting
\[
    w(t)
    =
    \frac{\pi}{2}
    \tan\left(\frac{\pi}{2}(1-t)\right),
\]
clamped in the implementation to avoid extreme weights.

\paragraph{AlphaGenome target features and kernel.}
For distributional calibration, we used AlphaGenome predictions as target
features. 
Because AlphaGenome requires a minimum input length of 2{,}048 base pairs, each DeepMEL2 enhancer sequence was inserted into the center of a fixed random DNA
background sequence before calling AlphaGenome. 
The same insertion procedure was used for both real target sequences and generated sequences, so the calibrated
features reflect changes in the enhancer region rather than changes in sequence context. 
The random background was generated once with base probabilities $(P_A = 0.2910, P_C = 0.2085, P_G = 0.2087, P_T = 0.2918)$ corresponding roughly to frequencies seen in regulatory DNA and then held fixed across feature evaluations.

For each sequence, we used the AlphaGenome ATAC predictions as enhancer activity features. The feature vector was defined by taking the maximum predicted ATAC value over genomic position for each retained ATAC channel, i.e. the peak height of each channel. 
We then applied a $\log(1+x)$ transform to the peak-height features, which made the marginal feature distributions closer to Gaussian and reduced the influence of very large peaks.

The target distribution for each cell topic was constructed from real DeepMEL2 sequences positive for that topic. We used up to 256 target sequences per topic when precomputing the AlphaGenome feature cache. During calibration, each update selects one topic and samples condition vectors from the training examples positive for that topic. 
Generated sequences are then compared to the corresponding topic-specific target feature distribution.

For the calibration objective, the log-transformed AlphaGenome features were projected using PCA fit on the training target features. 
We used 32 principal components with whitening; these components explained more than 99\% of the variance in the log-transformed target features. 
The PCA-whitened features were
used for the MMD calibration objective, while final reported feature metrics were computed on the non-PCA log-transformed AlphaGenome features.

For \methodname and \CGM, calibration was run for 20 epochs with batch size 128, 50 diffusion sampling steps, learning rate $10^{-4}$, cosine learning-rate decay to 10\% of the initial learning rate, and three random seeds. The final
validation evaluation generated 512 samples per condition.

\paragraph{Sampling method.}
For a fixed conditioning vector $\mbc$, the sampler starts from a
fully masked sequence $\mbx(0)$ of length $L$ and progressively unmasks tokens.
Let $\mbm$ denote the mask token and let
\[
    M(i-1) = \{j : \mbx(i-1)[j] = \mbm\}
\]
be the masked positions before step $i$.

The implementation uses a cosine unmasking schedule. For a nominal number of
sampling steps $T$, define
\[
    r(i) = \operatorname{round}\left(
        L \cos\left(\frac{\pi}{2}\frac{i}{T}\right)
    \right), \qquad i = 0,\ldots,T .
\]
The number of positions unmasked at step $i$ is
\[
    k_i = r(i-1) - r(i),
\]
with zero-sized steps omitted in the implementation. At each nonzero step:

\begin{enumerate}
    \item Select a subset $U(i) \subset M(i-1)$ with $|U(i)| = k_i$ uniformly
    at random without replacement.
    \item Evaluate the conditional model on the current partially
    masked sequence and condition vector $\mbc$. This gives logits
    $a_\theta^{(j)}(\cdot \mid \mbx(i-1), \mbc)$ over the four DNA bases at
    each position $j$.
    \item Convert logits to categorical probabilities,
    \[
        \pi_{\theta}^{(j)}(\cdot \mid \mbx(i-1), \mbc)
        =
        \operatorname{softmax}\left(
            a_\theta^{(j)}(\cdot \mid \mbx(i-1), \mbc)
        \right),
    \]
    and sample each newly unmasked base independently:
    \[
        \mbx(i)[j] \sim
        \pi_{\theta}^{(j)}(\cdot \mid \mbx(i-1), \mbc),
        \qquad j \in U(i).
    \]
\end{enumerate}

All other positions are left unchanged. Thus the chain begins fully masked and
ends at a fully specified DNA sequence.

\paragraph{Transition probabilities.}
Conditioned on $\mbc$, the initial distribution is deterministic:
\[
    \pi_0(\mbx(0)) =
    \mb{1}\{\mbx(0) \text{ is fully masked}\}.
\]
For step $i$, the transition probability factors into a position-selection term
and token-sampling terms:
\[
    p_\theta(\mbx(i) \mid \mbx(i-1), \mbc)
    =
    C_i
    \prod_{j \in U(i)}
    \pi_{\theta}^{(j)}
    \left(\mbx(i)[j] \mid \mbx(i-1), \mbc\right),
\]
where
\[
    C_i = \binom{|M(i-1)|}{k_i}^{-1}
\]
if the within-step unmasking set is treated as unordered. This constant accounts
for choosing which masked positions to unmask. It does not depend on $\theta$.
Since the number of masked positions is fixed by the schedule, $C_i$ is also the
same across trajectories of the same length.

\paragraph{Trajectory log-probability.}
As in the Genie 2 setting, the score-function gradient is computed using the full sampling trajectory
\[
    \mbx := \bigl(\mbx(0), \mbx(1), \ldots, \mbx(S)\bigr),
\]
where $\mbx(S)$ is the final generated DNA sequence and $S$ is the number
of unmasking steps.
The conditional trajectory log-probability is
\[
\begin{aligned}
    \logp_\theta(\mbx \mid \mbc)
    &=
    \log \pi_0(\mbx(0))
    +
    \sum_{i=1}^{S}
    \log p_\theta(\mbx(i) \mid \mbx(i-1), \mbc) \\
    &=
    \sum_{i=1}^{S} \log C_i
    +
    \sum_{i=1}^{S}
    \sum_{j \in U(i)}
    \log
    \pi_{\theta}^{(j)}
    \left(\mbx(i)[j] \mid \mbx(i-1), \mbc\right).
\end{aligned}
\]
We use only the second term, since the $\log C_i$ terms are constant with respect
to model parameters.

\paragraph{Parameter gradients.}
The gradient of the trajectory log-probability is therefore
\[
    \nabla_\theta \logp_\theta(\mbx \mid \mbc)
    =
    \sum_{i=1}^{S}
    \sum_{j \in U(i)}
    \nabla_\theta
    \log
    \pi_{\theta}^{(j)}
    \left(\mbx(i)[j] \mid \mbx(i-1), \mbc\right).
\]
During finetuning, trajectories are sampled without gradients, then each sampling step is replayed under the current policy to compute the corresponding log-probability contribution.
This lets the REINFORCE-style calibration gradient be accumulated one diffusion step at a time rather than attempting to backpropagate through the full sampling procedure.

\subsection{Overview figure: a smile made of SMILES}
We briefly outline the G2PT finetuning setup we use to produce \Cref{fig:overview}.

\paragraph{Features, kernel, and target features.}
For the smiley-face calibration experiment, we define the target distribution in a
two-dimensional standardized descriptor space,
\[
z = \bigl(z_{\mathrm{MolWt}}, z_{\mathrm{MolLogP}}\bigr).
\]
The standardization is computed from samples drawn from the pretrained base model.
We decode a presampled set of molecules, retain the valid ones, and compute the
coordinate-wise mean $\mu$ and standard deviation $\sigma$ of
$(\mathrm{MolWt}, \mathrm{MolLogP})$. A valid molecule with raw descriptor vector
$\mbx$ is then mapped to
\[
z = \frac{x - \mu}{\sigma}.
\]
Invalid molecules are assigned the placeholder value $(0,0)$. The synthetic smiley
target is constructed directly in this standardized space.

Given a target sample size $n$, we allocate
$n_{\mathrm{left}} = \lfloor 0.12 n \rfloor$ points to the left eye,
$n_{\mathrm{right}} = \lfloor 0.12 n \rfloor$ points to the right eye, and assign
the remaining points to the mouth. The two eyes are sampled as isotropic Gaussian clusters with
centers
\[
c_{\mathrm{left}} = (-0.8,\, 0.75),
\qquad
c_{\mathrm{right}} = (0.8,\, 0.75),
\]
and standard deviation $0.16$ in each coordinate. Equivalently,
\[
X_{\mathrm{left}} \sim \mathcal{N}\!\left(c_{\mathrm{left}},\, 0.16^2 I_2\right),
\qquad
X_{\mathrm{right}} \sim \mathcal{N}\!\left(c_{\mathrm{right}},\, 0.16^2 I_2\right).
\]

The mouth is sampled from a noisy lower arc. We first draw
\[
\theta \sim \mathrm{Unif}(210^\circ,\, 330^\circ),
\]
which traces the lower portion of a circle. We then form the corresponding point on
a radius-$1$ arc centered at $(0, 0.15)$,
\[
m(\theta) = (0,\,0.15) + (\cos\theta,\, \sin\theta),
\]
and add isotropic Gaussian perturbations with standard deviation $0.12$:
\[
X_{\mathrm{mouth}} = m(\theta) + \varepsilon,
\qquad
\varepsilon \sim \mathcal{N}(0,\, 0.12^2 I_2).
\]

After concatenating the two eye clouds and the mouth samples, we remove any point
whose Euclidean norm is smaller than $0.35$, i.e.,
\[
\|z\|_2 < 0.35.
\]
Rejected points are replaced by newly drawn samples until exactly $n$ target points
are obtained. This rejection step creates a small hole around the origin, which is
useful because invalid generated molecules are also mapped to $(0,0)$ under the
descriptor representation. The resulting empirical target distribution therefore
consists of two Gaussian eye clusters, a noisy lower arc for the smile, and an
empty region near the origin.

We use the energy-distance kernel for the MMD objective (\Cref{eq:energy_kernel}).

\paragraph{Training settings.}
We train for $500$ epochs with $\lambda = 10^{-4}$, $\alpha=1$, learning rate $2 \times 10^{-5}$, and batch size 2,048 split into four chunks to reduce the memory footprint.

\subsection{One-dimensional self-repulsion demonstration}\label{app_subsec:self_wt_demo}

\begin{figure}[h]
\centering
\includegraphics[width=0.6\linewidth]{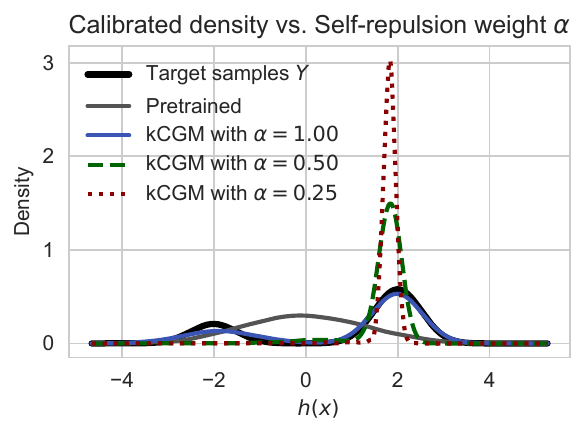}
\vspace{-0.8em}
\caption{
    Effect of self-repulsion weight $\alpha$ in a synthetic bimodal distribution-matching task.
    Smaller $\alpha$ weakens self-repulsion and produces more mode-seeking behavior.
}
\label{fig:self_wt_demo}
\end{figure}

To illustrate the role of the self-repulsion term in \methodname, we constructed a one-dimensional synthetic example using a learnable two-component Gaussian mixture model.
The target distribution is a symmetric bimodal mixture with modes at
 $-2$ and $+2$ and standard deviation $0.35$, while the
initial model has the same number of equally weighted mixture components but starts with means $-0.75, 0.75$ and standard deviations $1.0$.
We use the identity feature map and the energy-distance kernel.

Target samples are obtained by drawing 4,000 points from the fixed target mixture, and the initial model is separately sampled before calibration to visualize its pre-finetuning distribution.

The demo then runs \methodname several times from the same initial model, varying only the self-repulsion weight over the values
$\{0.25, 0.5, 0.75, 1.0\}$. 
In each run, the model is trained for $300$ epochs with batch size $256$, regularization parameter $\lambda=10^{-2}$, and learning rate $3\times 10^{-2}$. After calibration, the script draws fresh samples from each fitted model, reports the final logged loss values, and visualizes the resulting densities with one-dimensional kernel density estimates.

\section{Additional results}

\begin{figure}[ht]
\centering
\includegraphics[width=\textwidth]{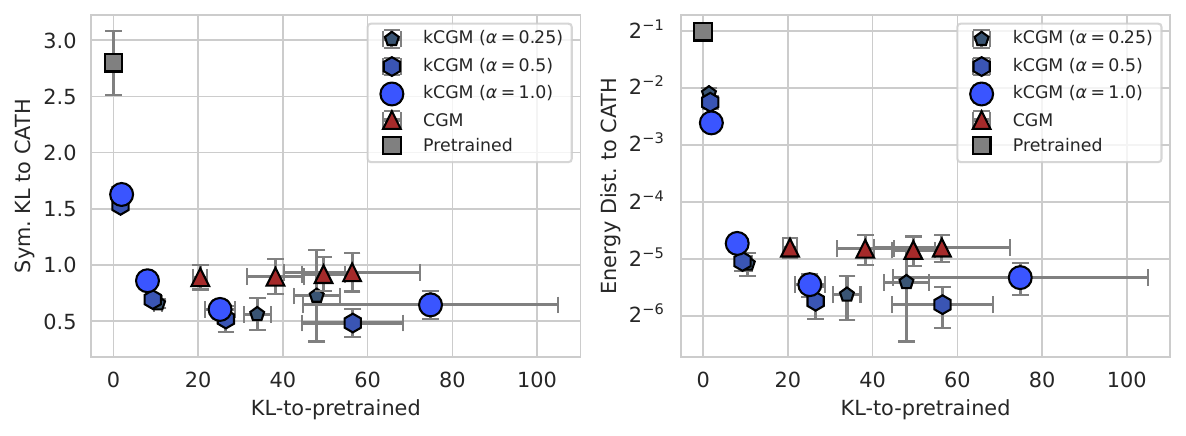}
    \caption{
Extended Genie 2 protein secondary-structure finetuning results for \methodname\ across all self-repulsion weights $\alpha$. The left panel reports symmetrized KL to the target CATH secondary-structure distribution, and the right panel reports energy distance to the same target distribution, both plotted against KL-to-pretrained. This figure extends \Cref{fig:genie2_results} by showing the full set of \methodname\ tradeoff curves obtained by varying $\alpha$.    }
    \label{fig:genie2_results_full}
\end{figure}

\subsection{Genie 2}
\Cref{fig:genie2_results_full} shows symmetrized KL and energy distance to the target CATH distribution of secondary structure proportions for each \methodname $(\lambda, \alpha)$ combination.
Both distribution-matching metrics result in the same ranking of finetuned models.

\subsection{G2PT}\label{app_subsec:g2pt_additional}

\Cref{fig:abx_other_features} repeats the experiment of \Cref{fig:abx_results} for scaffold and chemical descriptor features, showing that \methodname works for matching distributions from diverse feature extractors.

\begin{figure}
    \centering
    \includegraphics[width=\textwidth]{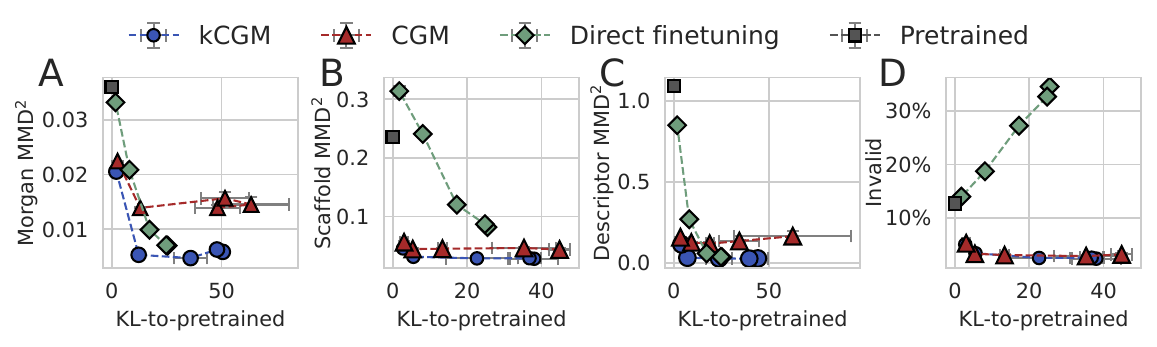}
    \caption{Comparison of direct finetuning, \methodname, and \CGM for adapting G2PT to the antibiotics
set for five values of the KL-regularization weight $\lambda$. Panels A–C evaluate separate \methodname and \CGM models trained to match the target Morgan fingerprint, Murcko scaffold, and descriptor distributions,
respectively, using the corresponding $\operatorname{MMD}^2$ metric. Panel D shows the fraction of chemically invalid samples for scaffold finetuned \methodname and \CGM and for direct finetuning. Direct finetuning does not depend on the feature target and is shared across all panels.}
    \label{fig:abx_other_features}
\end{figure}

\Cref{tab:g2pt_full} compares the methods for each target feature space with metrics interpolated to have equal KL-from-pretrained to evaluate a KL-constrained comparison.
Each \methodname model performs best on the features it was trained for, while transfer across feature spaces is asymmetric.
For example, Morgan finetuning transfers better to descriptor similarity than scaffold finetuning transfers to Morgan similarity, showing that the target feature spaces differ in granularity and therefore preserve different aspects of the antibiotic distribution.

\begin{table}[h]
\vspace{-1em}
\centering
\caption{
    Best metric values for runs with KL budget 25.
    Entries show linearly interpolated mean (SEM) at KL-to-pretrained equal to 25.
SEMs are interpolated by treating the two neighboring estimates as independent and
propagating variance under the linear interpolation weights. Bold indicates the best
method for each metric; underline indicates the second-best method.
}\label{tab:g2pt_full}
\begin{tabular}{lcccc}
\toprule
Method & Morgan MMD$^2$ & Scaffold MMD$^2$ & Descriptor MMD$^2$ & Validity \\
\midrule
kCGM Morgan & \textbf{0.005 (0.00017)} & \underline{0.039 (0.004)} & 0.15 (0.007) & 95\% (0.2\%) \\
kCGM Scaffold & 0.038 (0.00073) & \textbf{0.028 (0.0009)} & 1.2 (0.04) & \textbf{98\% (0.098\%)} \\
kCGM Descriptor & 0.024 (0.00028) & 0.13 (0.009) & \textbf{0.027 (0.0042)} & 96\% (0.2\%) \\
CGM Morgan & 0.014 (0.00019) & 0.15 (0.015) & 0.26 (0.012) & 92\% (0.21\%) \\
CGM Scaffold & 0.037 (0.00083) & 0.045 (0.00052) & 1.2 (0.054) & \underline{97\% (0.13\%)} \\
CGM Descriptor & 0.024 (0.00028) & 0.12 (0.0083) & 0.12 (0.0059) & 94\% (0.29\%) \\
Direct finetuning & \underline{0.007 (7e-05)} & 0.085 (0.0021) & \underline{0.036 (0.00049)} & 67\% (0.093\%) \\
\midrule
Pretrained & 0.036 (0.0002) & 0.24 (0.003) & 1.1 (0.00031) & 87\% (0.22\%) \\
\bottomrule
\end{tabular}
\end{table}

\section{Example generations}
\subsection{Genie 2}
We include uniform randomly selected samples from the pretrained model in \Cref{fig:genie_base_model_samples}, \CGM with its best $\lambda$ value in \Cref{fig:genie_CGM_model_samples}, and \methodname with its best $\lambda$, $\alpha$ combination in \Cref{fig:genie_kCGM_samples}.

\begin{figure}[h]
  \centering
  \includegraphics[width=\columnwidth]{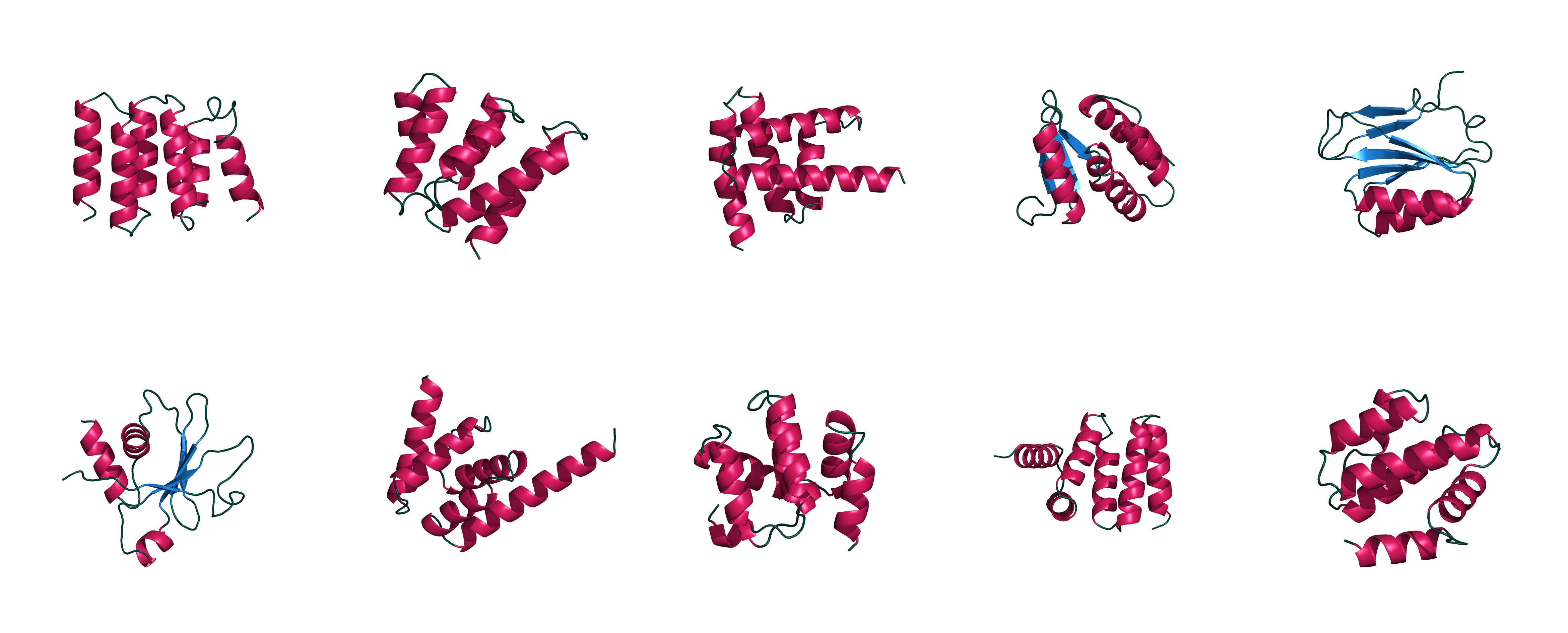}
  \caption{Genie 2 pretrained model samples.}
  \label{fig:genie_base_model_samples}
  \includegraphics[width=\columnwidth]{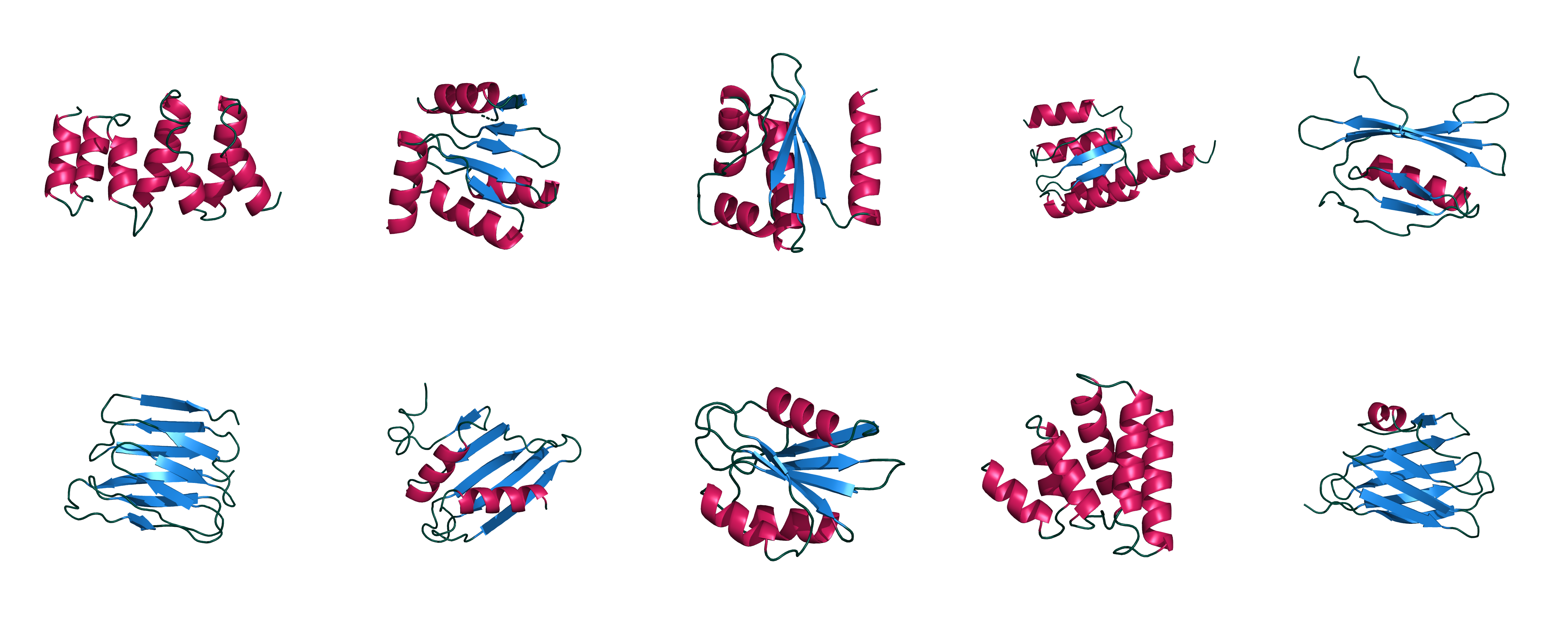}
  \caption{Genie 2 \CGM $\lambda=10^{-2}$.}
\label{fig:genie_CGM_model_samples}
  \includegraphics[width=\columnwidth]{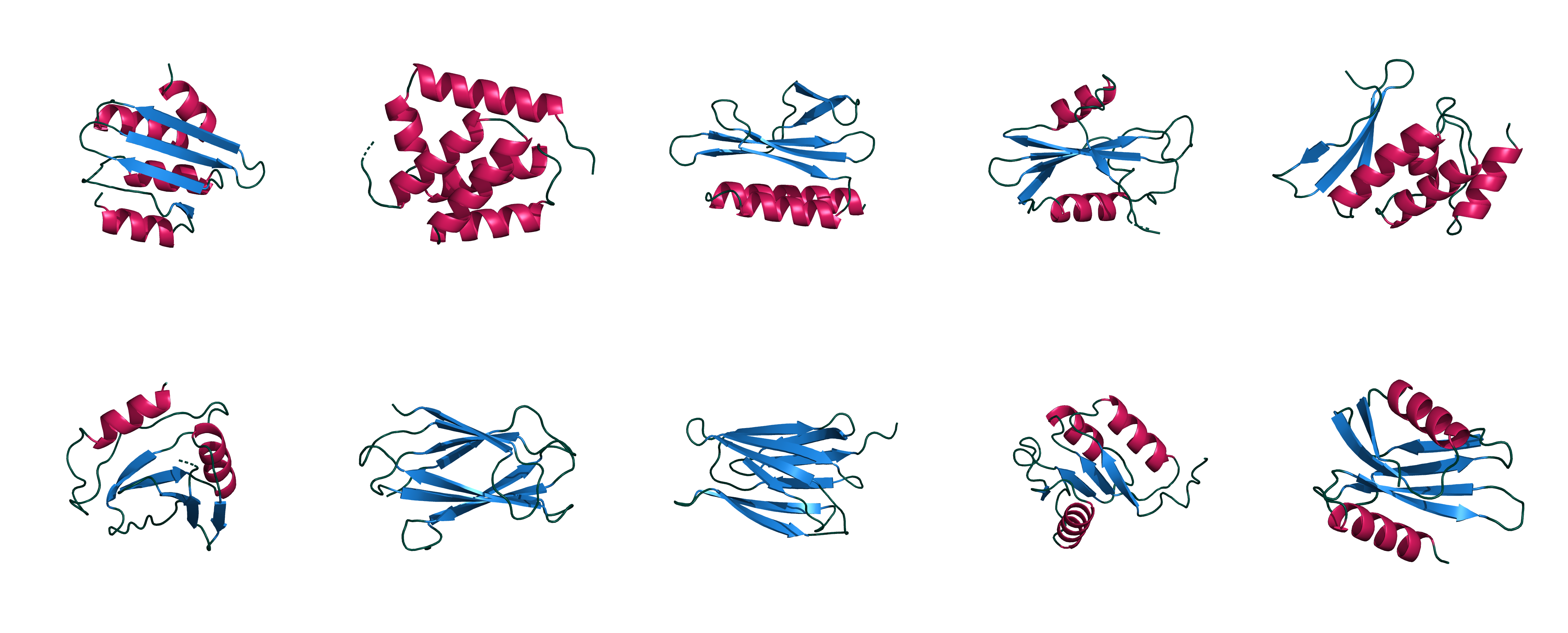}
  \caption{Genie 2 \methodname with $\lambda = 10^{-3}$, $\alpha = 0.25$.}
  \label{fig:genie_kCGM_samples}
  
\end{figure}

\subsection{G2PT}
We show example generations from the pretrained model (\Cref{fig:g2pt_base_model_samples}), the target antibiotics set (\Cref{fig:abx}) and from all the other finetuned models at regularization strength $\lambda=10^{-2}$.
\begin{itemize}[nosep]
    \item \Cref{fig:g2pt_kCGM_FP_model_samples}: \methodname with Morgan fingerprint features.
    \item \Cref{fig:g2pt_kCGM_scaffold_model_samples}: \methodname with generic Murcko scaffold features.
    \item \Cref{fig:g2pt_kCGM_desc_model_samples}: \methodname with descriptor features.
    \item \Cref{fig:g2pt_CGM_FP_model_samples}: \CGM with Morgan fingerprint features.
    \item \Cref{fig:g2pt_CGM_scaffold_model_samples}: \CGM with generic Murcko scaffold features.
    \item \Cref{fig:g2pt_CGM_desc_model_samples}: \CGM with descriptor features.
    \item \Cref{fig:g2pt_finetuned_model_samples}: Direct finetuning.
\end{itemize}

\begin{figure}[h]
  \centering
  
  \includegraphics[width=\columnwidth]{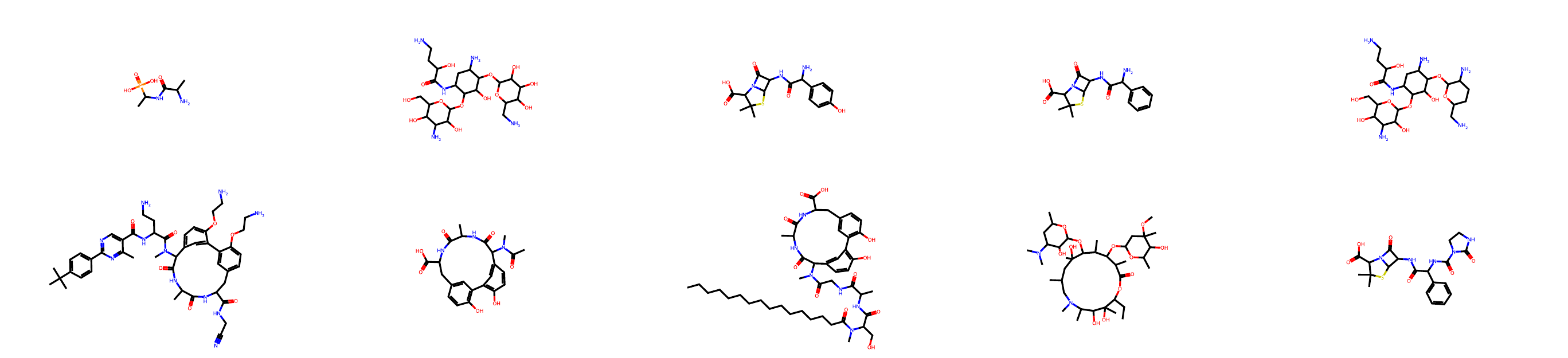}
  \caption{Antibiotics target set.}
  \label{fig:abx}
  
  \vspace{1em}
  
  \includegraphics[width=\columnwidth]{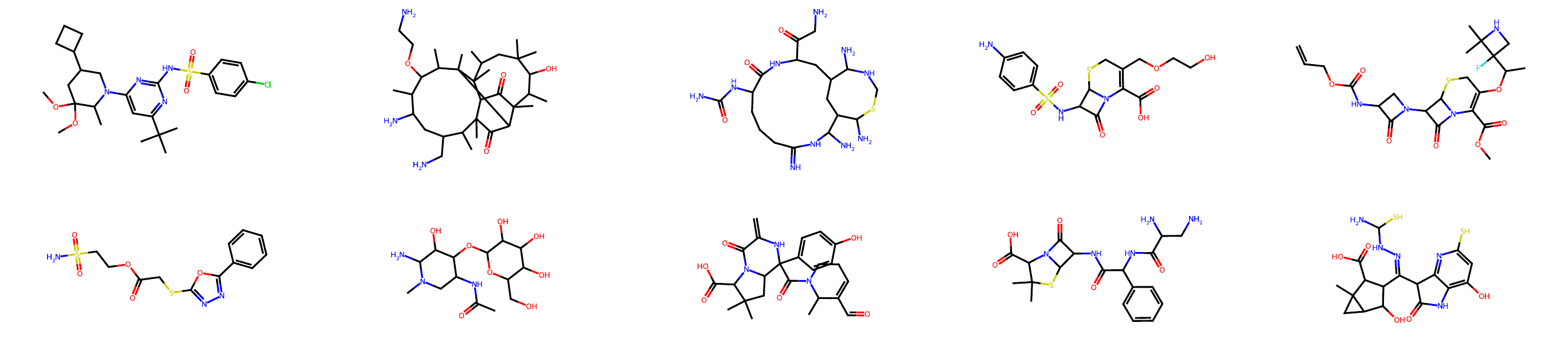}
  \caption{\methodname Morgan-FP tuned G2PT.}
  \label{fig:g2pt_kCGM_FP_model_samples}
  
  \vspace{1em}
  \includegraphics[width=\columnwidth]{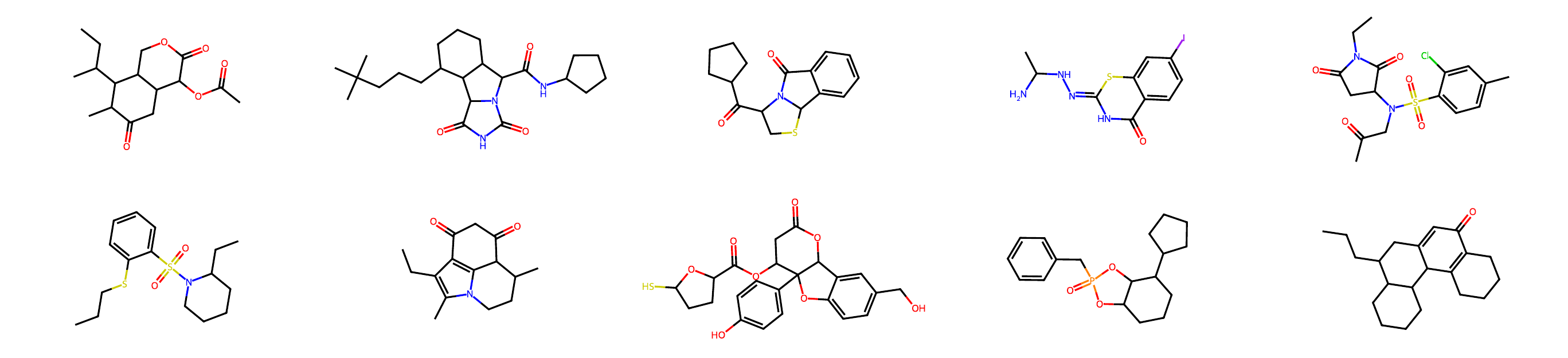}
  \caption{\methodname scaffold tuned G2PT.}
  \label{fig:g2pt_kCGM_scaffold_model_samples}

  \vspace{1em}
  \includegraphics[width=\columnwidth]{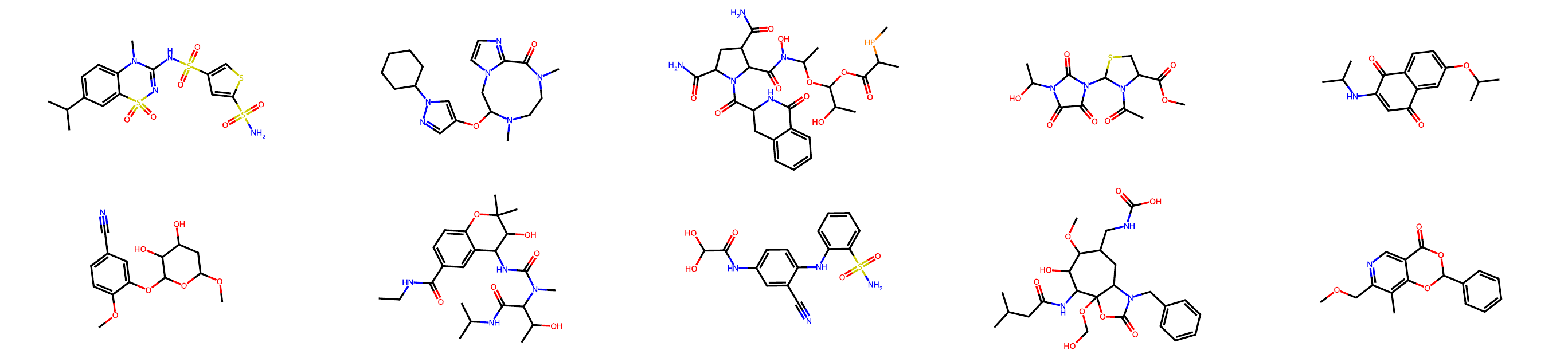}
  \caption{\methodname descriptor tuned G2PT.}
  \label{fig:g2pt_kCGM_desc_model_samples}
  
\end{figure}

\begin{figure}[t]
  \centering

  \includegraphics[width=0.9\columnwidth]{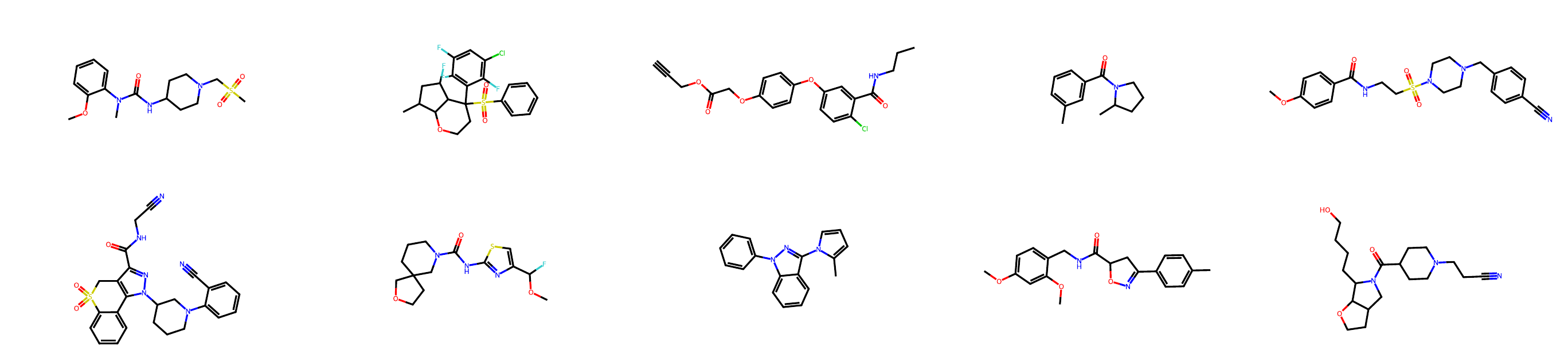}
  \caption{G2PT pretrained model samples.}
  \label{fig:g2pt_base_model_samples}
  
  \vspace{1em} 

  \includegraphics[width=0.9\columnwidth]{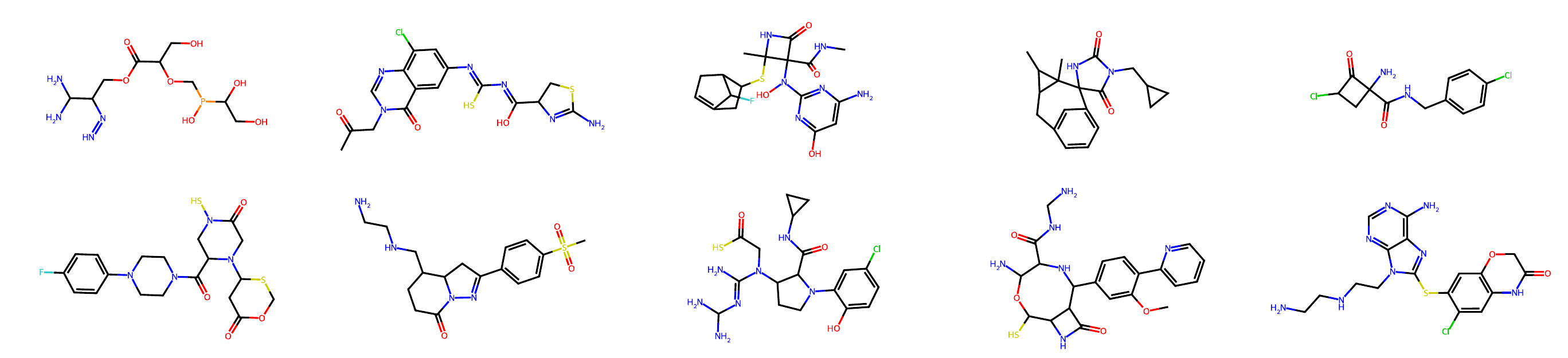}
  \caption{\CGM Morgan-FP tuned G2PT.}
  \label{fig:g2pt_CGM_FP_model_samples}
  
  \vspace{1em}
  
  \includegraphics[width=0.9\columnwidth]{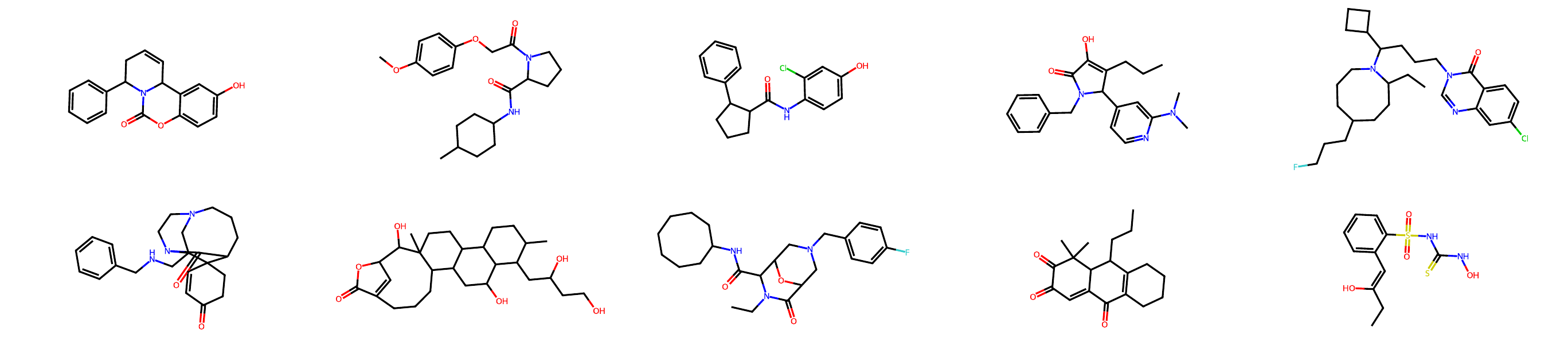}
  \caption{\CGM scaffold tuned G2PT.}
  \label{fig:g2pt_CGM_scaffold_model_samples}
  
  \vspace{1em}

  \includegraphics[width=0.9\columnwidth]{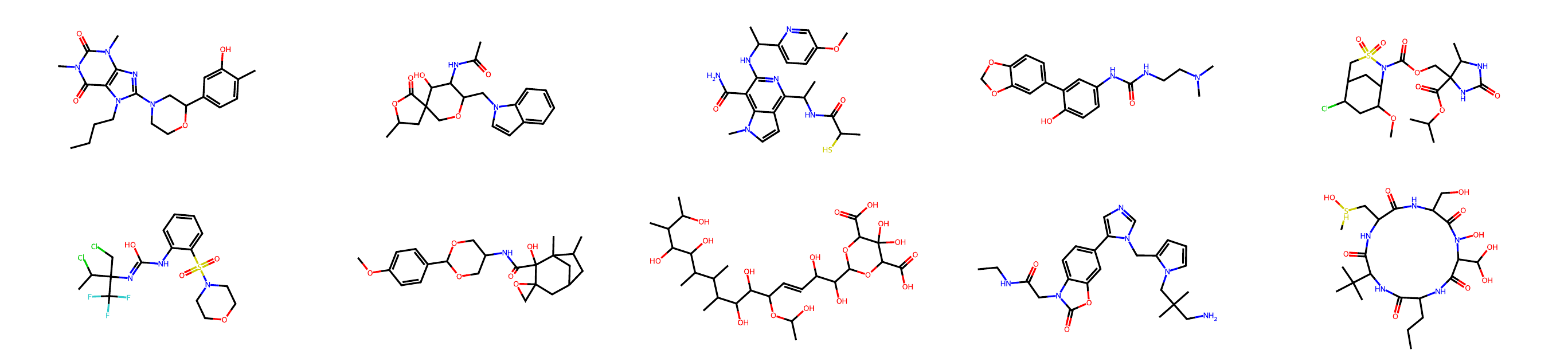}
  \caption{\CGM descriptor tuned G2PT.}
  \label{fig:g2pt_CGM_desc_model_samples}
  
  \vspace{1em}
  
  \includegraphics[width=0.9\columnwidth]{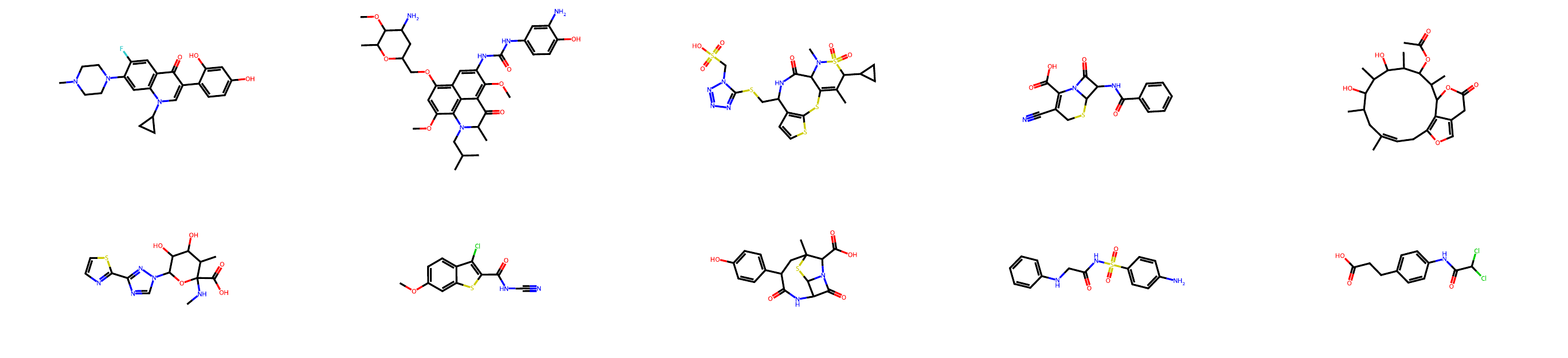}
  \caption{Direct finetuned G2PT.}
  \label{fig:g2pt_finetuned_model_samples}
\end{figure}



\end{document}